\documentclass{article}

\PassOptionsToPackage{numbers, compress}{natbib}

\usepackage[preprint]{neurips_2026}

\usepackage[utf8]{inputenc} 
\usepackage[T1]{fontenc}    
\usepackage[colorlinks,citecolor=blue,urlcolor=blue,linkcolor=red]{hyperref}
\usepackage{url}            
\usepackage{booktabs}       
\usepackage{amsfonts}       
\usepackage{nicefrac}       
\usepackage{microtype}      
\usepackage[table]{xcolor}  
\usepackage{graphicx}
\usepackage{subcaption}
\usepackage{amsmath}
\usepackage{amssymb}
\usepackage{multirow}
\usepackage{amsthm}
\newtheorem{theorem}{Theorem}[section]
\theoremstyle{remark}
\newtheorem{remark}[theorem]{Remark}
\usepackage{algorithm}
\usepackage{algorithmic}
\usepackage{wrapfig}
\usepackage{tabularx}
\usepackage{array}
\usepackage{appendix}
\usepackage{placeins}

\newcommand{\method}{ForesightFlow}
\newcommand{\E}{\mathbb{E}}
\newcommand{\R}{\mathbb{R}}
\newcommand{\KL}{D_{\mathrm{KL}}}
\newcommand{\sg}{\operatorname{sg}}

\newcolumntype{C}{>{\centering\arraybackslash}X}
\definecolor{softgreen}{RGB}{164,210,201}
\definecolor{softred}{RGB}{244,127,114}
\definecolor{mydarkblue}{RGB}{0,76,153}

\makeatletter
\renewcommand{\@noticestring}{%
  \noindent\rule{12pc}{0.4pt}\par\vspace{2.6pt}%
  \noindent $^{1}$Beijing Institute of Technology $^{2}$Tsinghua University $^{3}$The Hong Kong University of Science and Technology (Guangzhou)  $^{*}$Equal contribution $^{\dagger}$Corresponding author: \texttt{gangwang@bit.edu.cn}.%
}
\makeatother

\title{Potential-Guided Flow Matching for Vision-Language-Action Policy Improvement}

\author{
{\normalfont Yunpeng Mei$^{1,*}$ \quad
Jiakai He$^{1,*}$ \quad
Hongjie Cao$^{1,*}$ \quad
Chenyu Wang$^{1,*}$} \\
Xiaowen Zhu$^{1}$ \quad
Yihan Zhou$^{2}$ \quad
Jiamin Wang$^{1}$ \quad
Chenbo Xin$^{1}$ \quad
Peng Cheng$^{1}$ \\
Yuxuan Yang$^{1}$ \quad
Yijie Wang$^{1}$ \quad
Xinhu Zheng$^{3}$ \quad
Gao Huang$^{2}$ \quad
Jie Chen$^{1}$ \quad
Gang Wang$^{1,\dagger}$
}

\begin{document}

\maketitle

\allowdisplaybreaks

\begin{abstract}
Large vision-language-action (VLA) policies are increasingly trained as conditional generative models over action chunks. Yet deployment produces mixed-quality experience-successful demonstrations, partial completions, recoverable mistakes, and failures-that is difficult to use with standard imitation. Full behavior cloning (BC) imitates failures, filtered BC discards useful sub-trajectories, and offline reinforcement learning adds a large critic. We introduce \textbf{\method}, a self-guided flow-matching policy that augments each generated action chunk with a learned success-potential trajectory. The same flow proposes and scores candidate actions, enabling best-of-$K$ inference without an external critic. The key issue is that policy improvement and value calibration require different supervision: advantage weighting should emphasize high-quality actions, but applying the same weights to potential coordinates suppresses failure gradients and creates overconfident scores. We address this with decoupled advantage-weighted flow matching, applying exponentiated advantage weights only to action velocities while training potential velocities uniformly. We further derive a one-step boundary estimator for conditional flow matching, allowing advantage computation with a single stop-gradient forward pass. Across five BEHAVIOR-1K simulation tasks and five real-world bimanual tasks, \method~improves over imitation baselines, matches the strongest separate-critic baseline in simulation success, improves real-world success, and reduces training compute by $38\%$. Ablations show that decoupling prevents value hallucination, the one-step estimator preserves candidate-ranking fidelity, and self-guided sampling improves long-horizon execution.

\end{abstract}

\section{Introduction}

Large vision-language-action (VLA) models can parse language, fuse high-dimensional observations, and generate temporally extended action chunks for manipulation. Flow-matching policies are especially attractive because they combine supervised-training stability with expressive continuous action distributions and flexible sampling~\citep{pi0,recap,pi0.5,feng2026multi}. Yet deployed robots collect mixed-quality experience: demonstrations, partial completions, recoverable mistakes, and failures. Full behavior cloning imitates failures, filtered behavior cloning discards useful sub-trajectories, and offline reinforcement learning (RL) usually requires a large separate critic. This paper asks a simple question: \emph{can the generative policy itself provide the foresight normally supplied by a critic?} 

We answer yes with \textbf{\method}, a self-guided flow policy that augments each generated action chunk $a$ with a temporally aligned success-potential vector $s$. The same flow therefore proposes candidate actions and scores them, enabling best-of-$K$ inference without a standalone critic.

The central challenge is that action learning and potential learning require different supervision. Action updates should emphasize high-advantage behavior, while potential updates must keep failure samples visible for calibration. Applying the same advantage weights to both parts masks corrective gradients on overconfident failures. \method~avoids this with a decoupled objective: action velocities receive clipped exponentiated advantage weights, while potential velocities are trained uniformly on the full mixed-quality dataset. Figure~\ref{fig:framework} summarizes the training and inference pipeline.

\begin{figure*}[t]
    \vspace{-0.5cm}
    \centering
    \includegraphics[width=\linewidth]{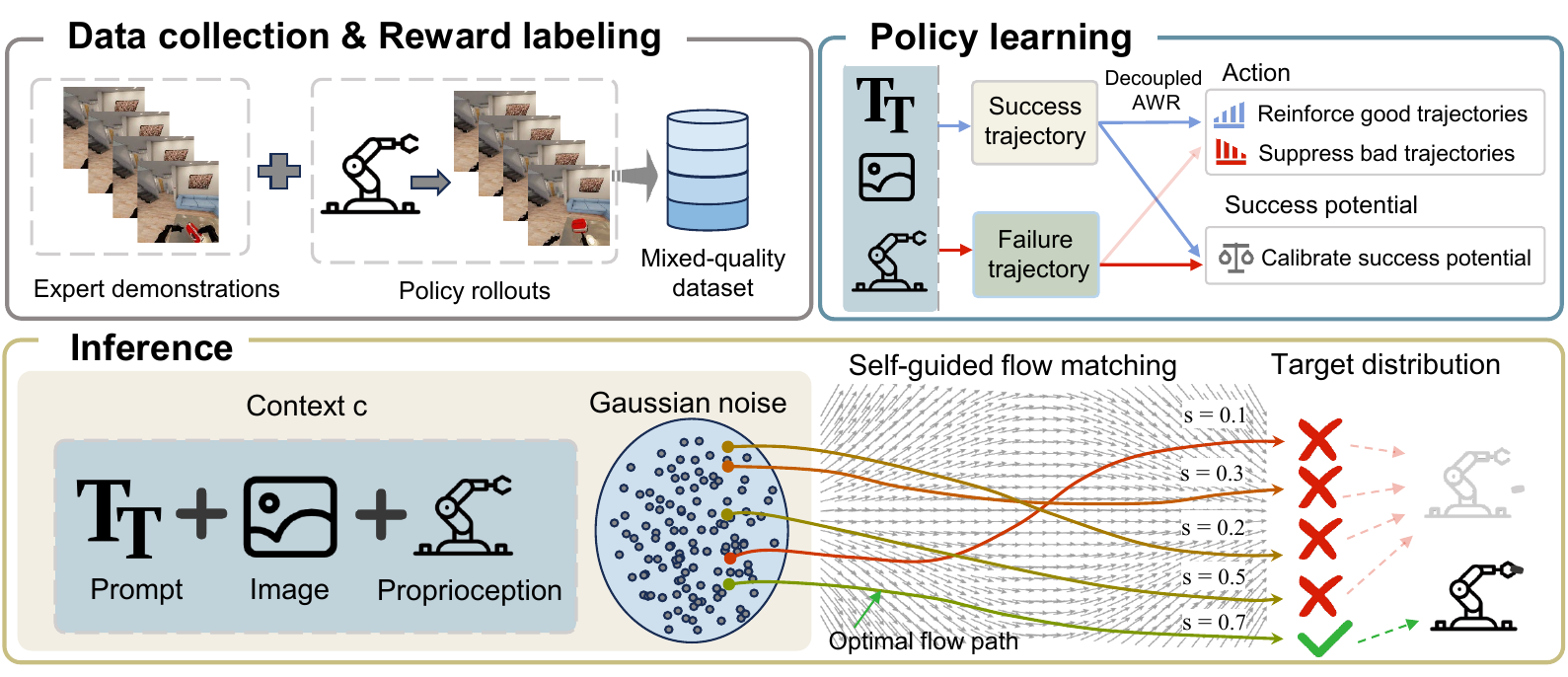}
    \caption{\textbf{\method~as a self-guided flow policy.} \method~augments each flow endpoint with action and success-potential coordinates. Decoupled training applies advantage weights only to action velocities while uniformly supervising potential velocities for calibration. At inference, generated potentials rank candidate action chunks without a separate critic.}
    \label{fig:framework}
    \vspace{-0.4cm}
\end{figure*}

A second challenge is efficiency. Instead of integrating the flow ODE for every training example, \method~uses a one-step baseline from a conditional-flow boundary identity under independent endpoint sampling. This justifies an action-independent context baseline for advantage-weighted regression (AWR); candidate ranking uses generated endpoint potentials, and we empirically validate the model's one-step ranking fidelity against high-NFE (number of function evaluations) references.

We evaluate \method~on five BEHAVIOR-1K simulation tasks and five real-world bimanual manipulation tasks. \method~achieves the best average simulation score, matches the strongest separate-critic baseline in average simulation success, improves real-world success, and reduces training compute by $38\%$. Ablations show that decoupling mitigates value hallucination, the single-step estimator preserves candidate rankings in our validation protocol, and self-guided sampling improves long-horizon performance.

\paragraph{Contributions.}
This paper makes three contributions.
\begin{enumerate}
    \item \textbf{A self-guided flow policy for VLA improvement.} We introduce \method, a conditional flow policy that augments action chunks with learned success potentials, enabling the same generative backbone to propose and score candidate actions.
    \item \textbf{Decoupled advantage-weighted flow matching.} We identify value hallucination as a failure mode of naively weighting all coordinates in a unified action-potential flow. We propose a decoupled objective that applies advantage weights only to action velocities while training potential velocities uniformly for calibration.
    \item \textbf{Efficient one-step foresight and best-of-$K$ inference.} We derive a boundary estimator for conditional flow matching, compute AWR baselines with one stop-gradient forward pass, and use generated potentials for best-of-$K$ inference.
\end{enumerate}

\section{Preliminaries}
\label{gen_inst}

\subsection{Sparse-Outcome Episodic Control}

We consider long-horizon manipulation as a finite-horizon partially observed Markov decision process
\begin{equation}
\mathcal M=(\mathcal S,\mathcal O,\mathcal A,P,R,\rho_0,T),
\end{equation}
where $z_t\in\mathcal S$ is the latent environment state, $o_t\in\mathcal O$ is the robot observation, $a_t\in\mathcal A\subset\R^{d_a}$ is the primitive action, $P$ is the transition kernel, $R$ is a sparse outcome signal, $\rho_0$ is the initial-state distribution, and $T$ is the horizon. The policy conditions on a multimodal context $c_t=(o_{\le t},\ell)$ consisting of observation history and a language instruction $\ell$.

Following recent generative robot policies, we use action chunking~\citep{act,diffusion_policy}. At decision time $t$, the policy predicts a flattened chunk
\begin{equation}
    a \equiv a_{t:t+H-1}\in\R^{H d_a},
\end{equation}
where $H$ is the chunk horizon. We use $t$ for environment time and $\sigma\in[0,1]$ for flow interpolation time. This distinction is important because \method~learns a continuous flow over action-potential endpoints, not over physical time.

\subsection{Conditional Flow Matching}

Conditional flow matching (CFM) learns a vector field that transports a simple prior $p_0=\mathcal N(0,I)$ to a context-conditioned data distribution $p_1(\cdot\mid c)$~\citep{flow_matching}. Let $x_0\sim p_0$ and $x_1\sim p_1(\cdot\mid c)$ denote independently sampled noise and data endpoints. The linear interpolation path is
\begin{equation}
    x_\sigma=(1-\sigma)x_0+\sigma x_1,\qquad u_\sigma=x_1-x_0.
\end{equation}
The learned velocity field $v_\theta$ is trained by squared regression to the conditional velocity as follows
\begin{equation}
    \mathcal L_{\rm CFM}(\theta)=
    \E_{c,x_0,x_1,\sigma}\!\left[\left\|v_\theta(x_\sigma,\sigma,c)-u_\sigma\right\|_2^2\right].
    \label{eq:cfm_loss}
\end{equation}
The population minimizer satisfies
\begin{equation}
    v^*(x_\sigma,\sigma,c)=\E[u_\sigma\mid x_\sigma,\sigma,c].
\end{equation}
This conditional-mean property is the basis of our one-step baseline estimator in Section~\ref{sec:method}; it also makes clear why the estimator should be interpreted as a population identity and empirically validated for finite networks.

\subsection{KL-Regularized Policy Improvement}

Offline policy improvement must be conservative because the dataset only covers a limited set of actions. A standard formulation maximizes value while penalizing deviation from the behavior distribution $\pi_\beta$:
\begin{equation}
    \max_\pi \, \E_{c\sim\mathcal D}\!\left[\E_{a\sim\pi(\cdot|c)} Q(c,a)
    -\tau\KL\!\left(\pi(\cdot|c)\,\|\,\pi_\beta(\cdot|c)\right)\right],
    \label{eq:kl_objective}
\end{equation}
where $Q(c,a)$ is an action-quality function and $\tau>0$ controls the strength of behavior regularization. The closed-form optimizer is the Boltzmann-reweighted behavior distribution
\begin{equation}
    \pi^*(a|c)=\frac{1}{Z(c)}\pi_\beta(a|c)\exp\!\left(\frac{Q(c,a)}{\tau}\right),
    \label{eq:optimal_policy}
\end{equation}
with normalization $Z(c)$. Advantage-weighted regression approximates this projection by upweighting high-advantage actions while retaining support near the data. \method~adapts this principle to flow matching, but applies it only to the action coordinates and not to the success-potential ones.

\section{Method: \method}
\label{sec:method}

\method~fine-tunes a VLA flow policy from mixed-quality data by combining four ingredients: i) an augmented action-potential endpoint, ii) stage-level success-potential targets, iii) decoupled advantage-weighted flow matching, and, iv) self-guided best-of-$K$ inference.

\subsection{Action-Potential Flow State}

A standard flow VLA policy models a conditional distribution over action chunks $a$. \method~extends the endpoint with a success-potential vector $s\in[0,1]^H$ aligned with the same horizon
\begin{equation}
    x = \begin{bmatrix} a \\ s \end{bmatrix} \in \R^{H(d_a+1)}.
\end{equation}
The potential coordinate is not a separate critic network. It is an additional generated coordinate of the same flow endpoint. Intuitively, each $s_k$ estimates the conditional likelihood that the corresponding local segment lies on a successful task trajectory under the dataset distribution. This gives every generated action chunk an intrinsic score. Unless otherwise stated, we aggregate the potential into a chunk-level score by temporal averaging~\citep{qchunk}
\begin{equation}
    \hat Q_\theta(c,a)=\frac{1}{H}\sum_{k=1}^H s_k,
    \label{eq:chunk_score}
\end{equation}
which is used as the endpoint candidate score during inference. During training, the boundary estimator below provides an action-independent context baseline $\hat V_{i,t}$ for AWR; this baseline is distinct from the generated endpoint score used for best-of-$K$ selection.

\subsection{Stage-Level Potential Targets}

Sparse binary episode outcomes are too coarse for long-horizon manipulation: a failed episode may still contain a good grasp, a correct approach, or a useful recovery. We therefore assign stage-level targets to each training chunk. For trajectory $i$ and time $t$, let
\begin{equation}
    y_{i,t}=\begin{cases}
        1, & \text{if the semantic stage assigned to chunk $(i,t)$ is completed,}\\
        0, & \text{otherwise.}
    \end{cases}
    \label{eq:stage_target}
\end{equation}
The endpoint target for the potential dimension is the broadcast vector $s_{1,i,t}=y_{i,t}\mathbf{1}_{H}\in\{0,1\}^{H}$ obtained by repeating this chunk-level stage label across the action horizon. This target construction keeps positive supervision for functional sub-trajectories even when the full rollout fails, while retaining negative supervision for calibration.

Given a current baseline estimate $\hat V_{i,t}$, we define the advantage
\begin{equation}
    A_{i,t}=y_{i,t}-\hat V_{i,t}.
    \label{eq:advantage}
\end{equation}
Successful chunks that the current model underestimates receive positive advantage and are emphasized in the action update. Failed chunks that the model overestimates receive negative advantage and are suppressed in the action update, but they remain essential for training the potential coordinate.

\subsection{One-Step Foresight for Advantage Estimation}

A naive baseline would require integrating the learned flow for every training context, which 
\begin{wrapfigure}[21]{r}{0.55\linewidth}
    \vspace{-1.0em}
    \centering
\includegraphics[width=\linewidth]{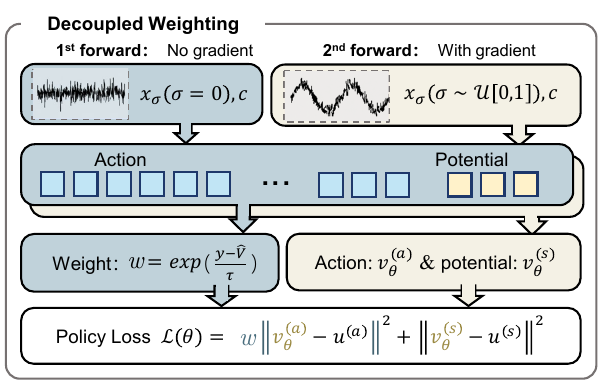}
    \caption{\textbf{Decoupled AWR in \method.} A stop-gradient boundary pass at $\sigma=0$ produces a one-step potential baseline and an advantage weight. A random-time CFM update then applies this weight only to the action velocity, while the potential velocity remains uniformly supervised on both successes and failures.}
\label{fig:decoupled_weighting}
\end{wrapfigure}
would add many NFEs per gradient step. \method~instead uses a boundary estimator derived from the CFM conditional-mean identity. For the linear path $x_\sigma=(1-\sigma)x_0+\sigma x_1$ with independent endpoints, the population-optimal velocity at $\sigma=0$ satisfies
\begin{equation}
    v^*(x_0,0,c)=\E[x_1\mid c]-x_0.
    \label{eq:boundary_identity_main}
\end{equation}
Thus a single Euler prediction from the noise endpoint estimates the conditional data mean
\begin{equation}
    \hat x_1 = x_0 + \sg\left(v_\theta(x_0,0,c)\right),
    \label{eq:single_step_estimator}
\end{equation}
where $\sg(\cdot)$ stops gradients through the baseline. We extract the success-potential component $\hat s_{i,t}$ from $\hat x_1$ and compute
\begin{equation}
    \hat V_{i,t}=\frac{1}{H}\sum_{k=1}^{H}\hat s_{i,t,k}.
    \label{eq:baseline}
\end{equation}
The estimator is exact at the population optimum under independent endpoint sampling. At that optimum, however, this boundary estimate is an action-independent context baseline rather than an action-conditioned candidate scorer. Candidate ranking at test time uses generated endpoint potentials after flow integration; when we use the one-step boundary proxy for ranking analysis, we treat it as an empirical property of the trained finite model rather than a consequence of the boundary identity. Figure~\ref{fig:decoupled_weighting} illustrates how the stop-gradient baseline and decoupled action-potential losses are combined in each update.

\subsection{Decoupled Advantage-Weighted Flow Matching}

The exponentiated advantage-weighted regression (AWR) weight is
\begin{equation}
w_{i,t}=\min\!\left(M,\exp\!\left(\frac{A_{i,t}}{\tau}\right)\right),
    \label{eq:awr_weight}
\end{equation}
where $\tau$ is the temperature and $M$ clips extreme weights. The scalar weight $w_{i,t}$ is broadcast across the action coordinates of the chunk; the potential velocity loss remains uniformly weighted over its full horizon vector. Let $v_\theta=(v_\theta^{(a)},v_\theta^{(s)})$ and $u=(u^{(a)},u^{(s)})$ denote the action and potential components of the predicted and target velocities. \method~optimizes
\begin{equation}
  \label{eq:decoupled_loss}
    \mathcal{L}(\theta)=\E\!\left[
    \underbrace{w_{i,t}\left\|v_\theta^{(a)}(x_\sigma,\sigma,c)-u^{(a)}_{i,t}\right\|_2^2}_{\text{advantage-weighted policy improvement}}
    +
    \underbrace{\left\|v_\theta^{(s)}(x_\sigma,\sigma,c)-u^{(s)}_{i,t}\right\|_2^2}_{\text{uniform potential calibration}}
    \right].
\end{equation}
The expectation is over $(i,t)\sim\mathcal D$, $x_0\sim p_0$, and $\sigma\sim\mathcal U[0,1]$.

This decoupling is the main difference between a unified action-potential flow and a naive joint actor-critic flow. The action coordinates should be selective: low-advantage actions should have little influence on the policy. The potential coordinates should be corrective: low-quality and failed samples must remain visible so that the model learns not to overestimate them. Applying $w_{i,t}$ to both parts removes corrective gradients on exactly the overconfident failures that matter most, leading to value hallucination. The decoupled objective retains the conservative AWR reweighting bias for actions while keeping the potential coordinate calibrated.

\subsection{Self-Guided Inference}

At test time, \method~uses the generated potential to guide its own samples. Given context $c_t$, we draw $K$ independent noise samples, integrate the learned flow, and obtain candidate endpoints $x^{(k)}=[a^{(k)};s^{(k)}]$. We then execute the candidate with the largest potential score
\begin{equation}
    k^*=\arg\max_{k\in\{1,\ldots,K\}} \frac{1}{H}\sum_{j=1}^{H}s^{(k)}_j,
    \qquad
    a_t^*=a_t^{(k^*)}.
    \label{eq:self_guided_inference}
\end{equation}
This best-of-$K$ rule is intentionally simple. It does not change the training distribution, does not require a separate critic forward pass, and uses the same action-potential coupling learned by the flow.

\section{Experiments}
\label{sec:experiments}

We evaluate whether \method~can turn mixed-quality experience into better VLA policies without the cost of a separate critic. The experiments are organized around four questions. \textbf{Q1:} Does \method~outperform full BC, filtered BC, and flow/offline-RL baselines on mixed-quality data? \textbf{Q2:} Does decoupled weighting prevent value hallucination? \textbf{Q3:} Does one-step foresight preserve the useful ranking information of higher-NFE integration? \textbf{Q4:} Does self-guided best-of-$K$ inference improve long-horizon execution?

\paragraph{Simulation tasks.}
We evaluate on five BEHAVIOR-1K tasks~\citep{behavior1k} in OmniGibson using the Galaxea R1 Pro humanoid embodiment, following VLA evaluation protocols~\citep{behavior1k, wbvima}. The tasks--Turning on radio, Picking up trash, Spraying fruit trees, Cook hot dogs, and Wash a baseball cap--span bimanual coordination, articulated-object interaction, mobile manipulation, and long-horizon sequencing. Episodes last $200$-$500$ seconds and terminate when the goal is satisfied or the horizon is exceeded. We report success rate over $100$ trials per task and a normalized Q-Score capturing stage-wise progress.

\paragraph{Real-world robot tasks.}
We further evaluate on the DexTeleop TeleAvatar Lite, a bimanual platform with dual $7$-DoF arms, custom parallel grippers, a head-mounted binocular perception module, and wrist RGB-D cameras. The controller runs at $20$ Hz. The five real-world tabletop tasks--Set-Paper-Roll, Pick-Trash, Cube-Stack, Transfer-Food, and Wipe-Whiteboard, cover precision placement, contact-rich manipulation, and multi-stage bimanual coordination with durations from $19$ to $138$ seconds. Figure~\ref{fig:task} summarizes the task suite; full scoring rules are in Appendix~\ref{app:task_definition}.

\begin{figure*}[t]
\includegraphics[width=\linewidth]{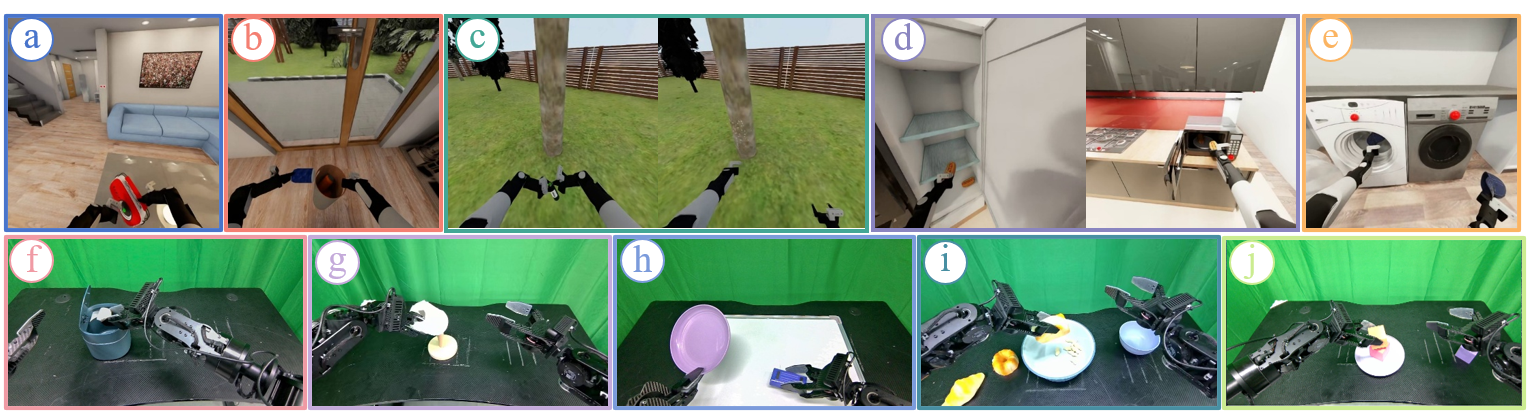}
    \vspace{-0.1em}
    \caption{\textbf{Evaluation suite.} \method~is evaluated in both simulation and the real world: (a) Turning on radio, (b) Picking up trash, (c) Spraying fruit trees, (d) Cook hot dogs, (e) Wash a baseball cap, (f) Set-Paper-Roll, (g) Pick-Trash, (h) Cube-Stack, (i) Transfer-Food, and (j) Wipe-Whiteboard. The tasks require long-horizon sequencing, bimanual coordination, contact-rich manipulation, and precision placement.}
    \label{fig:task}
    \vspace{-0.3cm}
\end{figure*}

\paragraph{Datasets.}
For each task, we construct a mixed-quality dataset with $200$ expert demonstrations and $100$ autonomous rollouts. Simulation expert data come from BEHAVIOR-1K demonstrations, while real-world expert trajectories are collected through VR teleoperation. Autonomous rollouts are generated by a pretrained $\pi_{0.5}$ checkpoint~\citep{pi0.5}. All trajectories are annotated with sparse stage-wise labels used to construct the binary targets in \eqref{eq:stage_target}. Details on dataset composition and frame-level success and failure statistics are provided in Appendix~\ref{dataset_discription}.

\subsection{Baselines}

We use matched backbones within each domain. In simulation, all methods fine-tune the $\pi_{0.5}$ architecture from a checkpoint pretrained on $50$ BEHAVIOR-1K tasks~\citep{solution}. In the real-world experiments, all methods use the $\pi_0$ architecture~\citep{pi0}. We compare against four baselines. \textbf{BC} trains the flow policy on the full mixed-quality dataset, measuring the cost of imitating failures. \textbf{Filtered BC} trains only on successful trajectories, measuring the cost of discarding partial but useful behavior. \textbf{IDQL} is a representative separate-critic offline RL baseline adapted to flow policies~\citep{idql}. \textbf{FQL} is a recent flow-aware offline RL method that distills a multi-step flow policy into a one-step student~\citep{fql}. These baselines isolate whether \method's benefit comes from using mixed data, from value-guided improvement, or from embedding the value signal directly in the flow.

\subsection{Main Results}
\paragraph{Simulation.}

\begin{table*}[t]
\vspace{1.5em}
\centering
\caption{\textbf{Simulation results on BEHAVIOR-1K.} We report normalized task score and success rate (SR, \%) on five long-horizon manipulation tasks. This table reports each method under its default deployment protocol: \method~uses its self-guided inference protocol with $K{=}5$ candidates; BC, Filtered BC, FQL, and IDQL use single-sample inference. Best entries are bolded.}\label{tab:sim_results}

\definecolor{softgreen}{RGB}{164, 210, 201}
\definecolor{softred}{RGB}{244, 127, 114}

\setlength{\tabcolsep}{3pt}
\footnotesize
\begin{tabularx}{\linewidth}{l|CCCCCCCCCC|CC}
\toprule

\multirow{2}{*}{{Method}} & 
\multicolumn{2}{c}{{Radio}} & 
\multicolumn{2}{c}{{Trash}} & 
\multicolumn{2}{c}{{Spray}} & 
\multicolumn{2}{c}{{Hotdog}} & 
\multicolumn{2}{c|}{{Wash}} & 
\multicolumn{2}{c}{{Average}} \\

\cmidrule(lr){2-3} 
\cmidrule(lr){4-5} 
\cmidrule(lr){6-7} 
\cmidrule(lr){8-9} 
\cmidrule(lr){10-11} 
\cmidrule(lr){12-13}

& Score  & SR & Score  & SR & Score  & SR & Score  & SR & Score  & SR & Score  & SR \\
\midrule

BC          & \cellcolor{softred!20}$0.42$ & \cellcolor{softred!20}$42.0$   & \cellcolor{softgreen!30}$0.46$ & \cellcolor{softgreen!35}$35.0$   & \cellcolor{softred!45}$0.02$ & \cellcolor{softred!45}$0.0$   & \cellcolor{softgreen!60}$\textbf{0.68}$ & \cellcolor{softgreen!60}$65.0$   & \cellcolor{softred!35}$0.19$ & \cellcolor{softred!40}$13.0$   & \cellcolor{softred!15}$0.35$ & \cellcolor{softred!10}$31.0$ \\


Filtered BC & \cellcolor{softgreen!15}$0.44$ & \cellcolor{softgreen!15}$44.0 $  & \cellcolor{softgreen!15}$0.43$ & \cellcolor{softgreen!15}$30.0$   & \cellcolor{softgreen!15}$0.18$ & \cellcolor{softgreen!15}$14.0$   & \cellcolor{softgreen!15}$0.51$ & \cellcolor{softgreen!15}$46.0$   & \cellcolor{softgreen!15}$0.33$ & \cellcolor{softgreen!15}$24.0 $  & \cellcolor{softgreen!15}$0.38$ & \cellcolor{softgreen!15}$31.6$ \\

FQL         & \cellcolor{softred!10}$0.42$ & \cellcolor{softred!10}$42.0$   & \cellcolor{softgreen!30}$0.47$ & \cellcolor{softgreen!20}$32.0$   & \cellcolor{softgreen!15}$0.19$ & \cellcolor{softgreen!15}$18.0$   & \cellcolor{softgreen!40}$0.62$ & \cellcolor{softgreen!40}$59.0$   & \cellcolor{softgreen!10}$0.24$ & \cellcolor{softgreen!15}$21.0$   & \cellcolor{softgreen!15}$0.39$ & \cellcolor{softgreen!20}$34.4$ \\

IDQL        & \cellcolor{softgreen!30}$0.47$ & \cellcolor{softgreen!30}$47.0$   & \cellcolor{softgreen!40}$0.48$ & \cellcolor{softgreen!40}$36.0$   & \cellcolor{softgreen!30}$0.23$ & \cellcolor{softgreen!35}$\textbf{21.0}$   & \cellcolor{softgreen!60}$\textbf{0.68}$ & \cellcolor{softgreen!60}$\bf 66.0$   & \cellcolor{softgreen!20}$\textbf{0.34}$ & \cellcolor{softgreen!20}$25.0$   & \cellcolor{softgreen!40}$0.44$ & \cellcolor{softgreen!50}$39.0$ \\

{\method~(ours)} & \cellcolor{softgreen!50}$\textbf{0.51}$ & \cellcolor{softgreen!50}$\textbf{51.0}$ & \cellcolor{softgreen!70}$\textbf{0.57}$ & \cellcolor{softgreen!50}$\textbf{38.0}$ & \cellcolor{softgreen!50}$\textbf{0.28}$ & \cellcolor{softgreen!30}$20.0$ & \cellcolor{softgreen!45}$0.62$ & \cellcolor{softgreen!45}$59.0$ & \cellcolor{softgreen!20}$\textbf{0.34}$ & \cellcolor{softgreen!40}$\textbf{30.0}$ & \cellcolor{softgreen!50}$\textbf{0.46}$ & \cellcolor{softgreen!55}$\textbf{39.6}$ \\
\bottomrule
\end{tabularx}
\end{table*}

Table~\ref{tab:sim_results} shows that the mixed-quality setting exposes the weaknesses of standard imitation. Full BC preserves data coverage but can imitate failures, collapsing on Spray ($0\%$ SR). Filtered BC improves some precision behaviors but loses useful sub-trajectories, for example underperforming full BC on Cook Hot Dogs. \method~achieves the best average score ($0.46$) and matches the separate-critic IDQL baseline in average success rate ($39.6\%$ vs $39.0\%$), while outperforming the baselines on Radio, Trash, and Wash. To isolate candidate budget, Table~\ref{tab:mc_ranking} compares IDQL and \method~at both $K{=}1$ and $K{=}5$: \method~is weaker than IDQL in average SR at $K{=}1$ ($34.6\%$ vs $39.0\%$), but remains ahead at $K{=}5$ ($39.6\%$ vs $38.4\%$). Relative to IDQL, the main advantage is not only performance but simplicity: the same flow that samples actions also supplies the score used for candidate selection.
FQL is a strong flow-RL reference, but in high-dimensional $\pi_{0.5}$ action-chunk setting ($23$ DoF and $H{=}30$), one-step distillation does not surpass IDQL or \method.

\begin{table*}[t]
\centering
\caption{\textbf{Real-world results.} We report normalized task score and full-horizon success rate (SR, \%) across five bimanual manipulation tasks. \method~achieves the best average score and success rate while avoiding a separate critic network.}
\label{tab:real_results}

\definecolor{softgreen}{RGB}{164, 210, 201}
\definecolor{softred}{RGB}{244, 127, 114}

\setlength{\tabcolsep}{3pt}
\footnotesize
\begin{tabularx}{\linewidth}{l|CCCCCCCCCC|CC}
\toprule
\multirow{2}{*}{{Model}} & 
\multicolumn{2}{c}{{Set-Paper-Roll}} & 
\multicolumn{2}{c}{{Pick-Trash}} & 
\multicolumn{2}{c}{{Cube-Stack}} & 
\multicolumn{2}{c}{{Transfer-Food}} & 
\multicolumn{2}{c|}{{Whiteboard}} & 
\multicolumn{2}{c}{{Average}} \\
\cmidrule(lr){2-3} 
\cmidrule(lr){4-5} 
\cmidrule(lr){6-7} 
\cmidrule(lr){8-9} 
\cmidrule(lr){10-11} 
\cmidrule(lr){12-13}
& Score & SR & Score & SR & Score & SR & Score & SR &Score & SR & Score & SR \\
\midrule

Filtered BC          & \cellcolor{softgreen!15}$0.52$  & \cellcolor{softgreen!15}$23.0$   & \cellcolor{softgreen!15}$0.36$  & \cellcolor{softgreen!15}$38.0$   & \cellcolor{softgreen!15}$0.62$  & \cellcolor{softgreen!15}$19.0$   & \cellcolor{softgreen!15}$0.57 $ & \cellcolor{softgreen!15}$29.0$   & \cellcolor{softgreen!15}$0.47$ & \cellcolor{softgreen!15}$14.0 $  & \cellcolor{softgreen!15}$0.51$ & \cellcolor{softgreen!15}$24.6$  \\

IDQL                 & \cellcolor{softgreen!35}$0.59$  & \cellcolor{softgreen!40}$33.0$   & \cellcolor{softgreen!40}$\textbf{0.48}$  & \cellcolor{softgreen!35}$46.0 $  & \cellcolor{softgreen!25}$0.64$  & \cellcolor{softgreen!25}$23.0$   & \cellcolor{softgreen!40}$0.66$  & \cellcolor{softgreen!45}41.0   & \cellcolor{softgreen!40}0.57 & \cellcolor{softgreen!35}\textbf{20.0}   & \cellcolor{softgreen!40}0.59 & \cellcolor{softgreen!40}32.6  \\

{\method~(ours)}  & \cellcolor{softgreen!50}$\textbf{0.62}$  & \cellcolor{softgreen!50}$\textbf{36.0}$   & \cellcolor{softgreen!40}$\textbf{0.48}$  & \cellcolor{softgreen!60}$\textbf{51.0}$   & \cellcolor{softgreen!60}$\textbf{0.72}$  & \cellcolor{softgreen!40}$\textbf{26.0} $  & \cellcolor{softgreen!60}$\textbf{0.69}$  & \cellcolor{softgreen!70}$\textbf{45.0}$   & \cellcolor{softgreen!50}$\textbf{0.59}$ & \cellcolor{softgreen!30}$19.0$   & \cellcolor{softgreen!55}$\textbf{0.62}$ & \cellcolor{softgreen!55}$\textbf{35.4}$  \\ \bottomrule
\end{tabularx}
\end{table*}

\paragraph{Real-world deployment.}
Table~\ref{tab:real_results} shows that \method~also improves real-world manipulation. It reaches the best average score ($0.62$) and success rate ($35.4\%$), outperforming Filtered BC by $10.8$ percentage points and IDQL by $2.8$ percentage points in average SR. The gains are largest on tasks where partial progress matters: Cube-Stack benefits from calibrated ranking of precise placements, and Transfer-Food benefits from selecting action chunks that preserve multi-stage progress. Figure~\ref{fig:real_score} provides the stage-wise breakdown.

\begin{figure}[t]
    \centering
    \includegraphics[width=\linewidth]{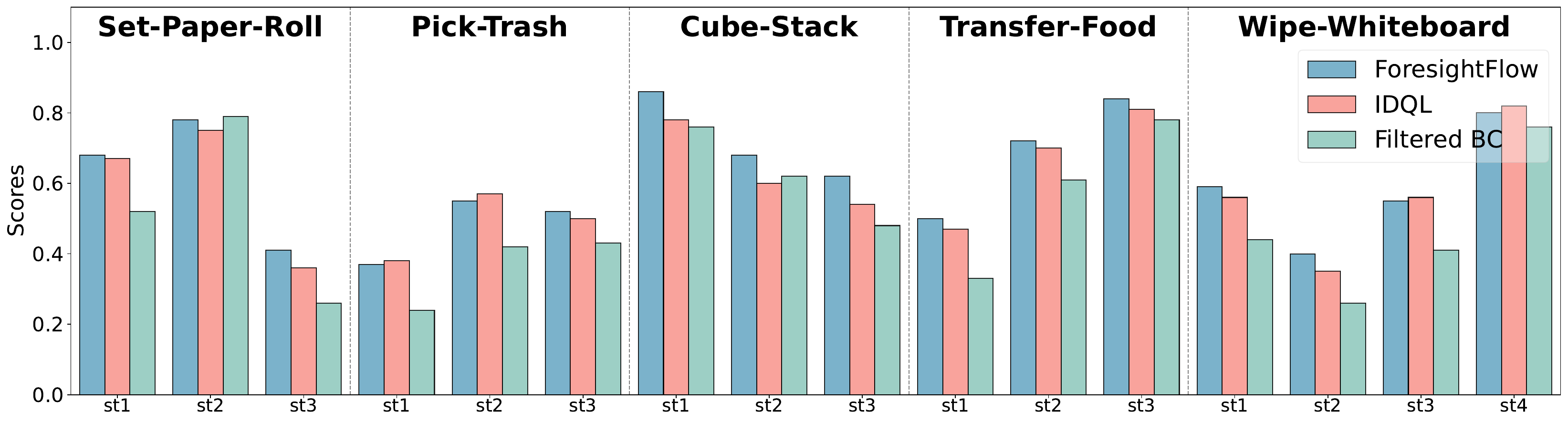}
    \caption{\textbf{Stage-wise real-world progress.} Normalized completion score for each stage across five bimanual tasks. \method~improves most on stages where failures are recoverable and partial trajectories provide useful supervision.}
    \label{fig:real_score}
\end{figure}

\begin{wraptable}[11]{r}{0.55\linewidth}
    \vspace{-0.9em}
    \centering
    \caption{\textbf{Computational efficiency.} \method~matches IDQL-level performance while reducing training compute and removing the separate critic network. ``Critic params'' counts only additional value-related parameters.}
    \label{tab:efficiency}
    \vspace{0.0cm}
    \footnotesize
    \setlength{\tabcolsep}{1.8pt}
    \renewcommand{\arraystretch}{0.90}
    \resizebox{\linewidth}{!}{%
    \begin{tabular}{@{}l|c|cc|cc@{}}
    \toprule
    & \multicolumn{1}{c|}{Training} & \multicolumn{2}{c|}{Critic} & \multicolumn{2}{c}{Latency} \\
    Method & GPU hrs & Params & Total & $K{=}1$ & $K{=}5$ \\
    \midrule
    IDQL & $287$ & $\sim$500\,M & $\sim$2.84\,B & \textbf{119\,ms} & 183\,ms \\
    \method~(ours) & \textbf{$178$} & \textbf{$\sim$1\,K} & \textbf{2.35\,B} & 127\,ms & \textbf{155\,ms} \\
    \bottomrule
    \end{tabular}%
    }
    \vspace{-0.8em}
\end{wraptable}

\paragraph{Compute and parameter efficiency.}
Table~\ref{tab:efficiency} compares \method~with IDQL on the simulation backbone. IDQL requires a sequential critic-pretraining stage followed by actor fine-tuning, totaling $287$ GPU hours. \method~uses a single joint fine-tuning stage and $178$ GPU hours, yielding a $38\%$ reduction. The parameter count differs sharply: \method~adds only a $\sim$1K-parameter progress projection head, while IDQL uses a $\sim$500M-parameter separate critic. At $K{=}5$, \method~is faster at inference because candidate scoring is produced jointly with action generation.

\subsection{Ablation Studies}
\label{sec:ablation}

\paragraph{Decoupled weighting prevents value hallucination.}

\begin{wraptable}[8]{r}{0.55\linewidth}
    \vspace{-0.9em}
    \centering
    \caption{Stage-wise task completion rates (\%) on Turning on Radio.}
    \label{tab:stage_ablation}
    \vspace{-0.0cm}
    \small
    \resizebox{\linewidth}{!}{%
    \begin{tabular}{@{}l|ccc@{}}
    \toprule
    {Model} & {Stage $1$} & {Stage $2$} & {Stage $3$} \\
    \midrule
    BC baseline                & $88.0$ & $59.0$ & $44.0$ \\
    \method-coupled            & $92.0$ & {$\bf 62.0$} & $42.0$ \\
    {\method-decoupled (ours)} & {$\bf 97.0$} & $60.0$ & {$\bf 51.0$} \\
    \bottomrule
    \end{tabular}%
    }
    \vspace{-0.8em}
\end{wraptable}

We analyze per-step potential predictions across $19{,}905$ frames from five mixed-quality Radio rollouts. Figure~\ref{fig:coupled} shows that the coupled baseline assigns high potentials to explicit failures such as knocking the radio over or slipping. In contrast, decoupled training lowers these predictions while preserving high scores for successful segments. The stage-wise evaluation in Table~\ref{tab:stage_ablation} shows the downstream effect: the coupled model improves early approach/grasp behavior but degrades at final actuation, whereas \method~maintains improvement through the final stage.

\paragraph{One-step foresight preserves useful rankings.}

\begin{wrapfigure}[18]{r}{0.55\linewidth}
    \vspace{-0.7em}
    \centering
    \includegraphics[width=\linewidth]{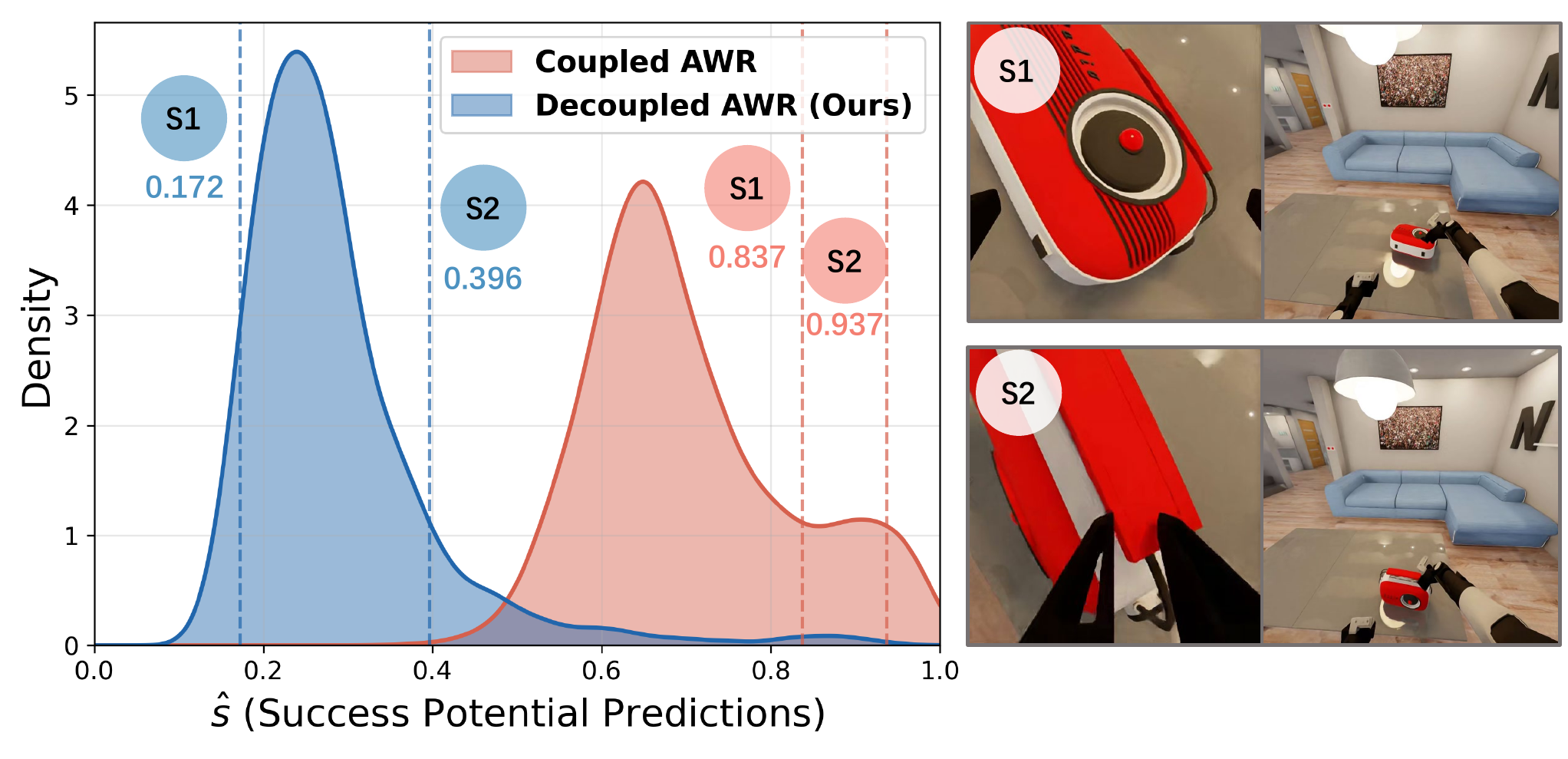}
    \caption{\textbf{Value calibration on mixed-quality rollouts.} Coupled AWR overestimates failures because the potential loss is down-weighted on negative samples. Decoupled AWR keeps failure supervision active and produces calibrated potentials.}
    \label{fig:coupled}
    \vspace{0.5cm}
\end{wrapfigure}

For training efficiency, we compare the one-step estimator (NFE$\,=\,$$1$) with multi-step ODE integration (NFE$\,=\,$$20$). NFE$\,=\,$$1$ requires $8.4$ hours per $5{,}000$ steps versus $11$ hours for NFE$\,=\,$$20$, and both reach comparable success by $15$k steps. For ranking fidelity, we use NFE$\,=\,$$100$ as a high-NFE reference. Under a held-out within-observation protocol with $K=8$ candidates, Kendall $\tau$ between NFE$\,=\,$$1$ and NFE$\,=\,$$100$ rankings ranges from $0.801$ to $0.863$ across tasks, with mean top-$1$ agreement of $87.2\%$ and mean MC gain ratio of $96.8\%$ (Table~\ref{tab:nfe_cross_task}). MC gain denotes the fraction of the NFE$\,=\,$$100$ best-of-$K$ improvement retained when candidates are ranked by the NFE$\,=\,$$1$ estimator.

\paragraph{Self-guided sampling improves long-horizon performance.}
Table~\ref{tab:mc_ranking} compares $K{=}1$ and $K{=}5$ candidate ranking. \method's average SR improves from $34.6\%$ to $39.6\%$, with the largest gains on multi-stage tasks such as Radio, Trash, and Hotdog. IDQL is comparatively flat under the same protocol. We do not claim this is a universal property of separate critics; rather, it suggests that in this setting, coupling candidate generation and potential prediction inside the same flow produces a useful ranking signal for best-of-$K$ inference.

\begin{table}[!h]
\centering
\caption{\textbf{Cross-task NFE$\,=\,$$1$ vs NFE$\,=\,$$100$ ranking fidelity.} For each held-out observation, $K=8$ candidates are scored with both estimators. The one-step estimator largely preserves the candidate ranking used by self-guided inference. The MC gain measures the fraction of the NFE$\,=\,$$100$ best-of-$K$ improvement retained by NFE$\,=\,$$1$ ranking, and $\kappa$ reports trajectory curvature in the success-potential subspace.}
\label{tab:nfe_cross_task}
\vspace{0.2cm}
\small
\setlength{\tabcolsep}{4pt}
 \begin{tabularx}{\linewidth}{@{}l*{4}{>{\centering\arraybackslash}X}c@{}}
\toprule
Task & Kendall $\tau$ & Top-$1$ agree & Pairwise acc & MC gain & $\kappa$ med, ~~ frac($<0.1$) \\
\midrule
Radio        & 0.863 & 93.8\% & 93.1\% & 98.5\% & 0.026, ~~99.6\% \\
Hot Dogs     & 0.801 & 84.4\% & 90.1\% & 96.7\% & 0.028, ~~98.0\% \\
Trash        & 0.835 & 82.8\% & 91.7\% & 95.3\% & 0.039, ~~90.2\% \\
Baseball Cap & 0.818 & 84.4\% & 90.9\% & 95.5\% & 0.030, ~~96.9\% \\
Fruit Trees  & 0.853 & 90.6\% & 92.6\% & 97.9\% & 0.027, ~~96.5\% \\
\midrule
Mean         & 0.834 & 87.2\% & 91.7\% & 96.8\% & 0.030,~ ~96.2\% \\
\bottomrule
\end{tabularx}
\end{table}

\begin{table*}[t]
\centering
\caption{\textbf{MC ranking ablation across five simulation tasks ($K{=}1$ vs $K{=}5$).} \method's success rate improves with MC ranking by $+5.0$ pp on average, while IDQL's is essentially flat. Score/SR (\%) reported per task. Boldface marks the best within each method group (IDQL $K{=}1$ vs $K{=}5$; \method~$K{=}1$ vs $K{=}5$).}
\label{tab:mc_ranking}
\setlength{\tabcolsep}{3pt}
\footnotesize
\begin{tabularx}{\linewidth}{@{}l|CCCCCCCCCC|CC@{}}
\toprule
\multirow{2}{*}{Method} &
\multicolumn{2}{c}{Radio} & \multicolumn{2}{c}{Trash} & \multicolumn{2}{c}{Spray} &
\multicolumn{2}{c}{Hotdog} & \multicolumn{2}{c|}{Wash} & \multicolumn{2}{c}{Average} \\
\cmidrule(lr){2-3} \cmidrule(lr){4-5} \cmidrule(lr){6-7} \cmidrule(lr){8-9} \cmidrule(lr){10-11} \cmidrule(lr){12-13}
& Score & SR & Score & SR & Score & SR & Score & SR & Score & SR & Score & SR \\
\midrule
IDQL ($K{=}1$) & 0.47 & 47.0 & \textbf{0.48} & \textbf{36.0} & \textbf{0.23} & \textbf{21.0} & \textbf{0.68} & \textbf{66.0} & 0.34 & 25.0 & \textbf{0.44} & \textbf{39.0} \\
IDQL ($K{=}5$) & \textbf{0.49} & \textbf{49.0} & 0.47 & 35.0 & 0.21 & 18.0 & 0.65 & 64.0 & \textbf{0.35} & \textbf{26.0} & 0.43 & 38.4 \\
\midrule
\method~($K{=}1$)  & 0.45 & 45.0 & 0.55 & 28.0 & \textbf{0.28} & \textbf{21.0} & 0.54 & 53.0 & 0.31 & 26.0 & 0.43 & 34.6 \\
\method~($K{=}5$)  & \textbf{0.51} & \textbf{51.0} & \textbf{0.57} & \textbf{38.0} & \textbf{0.28} & 20.0 & \textbf{0.62} & \textbf{59.0} & \textbf{0.34} & \textbf{30.0} & \textbf{0.46} & \textbf{39.6} \\
\bottomrule
\end{tabularx}
\end{table*}

\newpage
\section{Related Work}

\paragraph{VLA models and generative robot policies.}
Large VLA models~\citep{pi0,rt2,openvla,gr00t,rdt,vpp} combine broad visual-language pretraining with robot action prediction, enabling general-purpose manipulation across tasks and embodiments. Early systems often discretized actions~\citep{rt2,rt1}; more recent policies use continuous generative models, including diffusion policies~\citep{diffusion_policy} and flow-matching policies~\citep{pi0,flow_matching,rectified_flow}. These models provide expressive multimodal action distributions, but are still commonly improved through supervised imitation on curated demonstrations. \method~targets the next stage: improving such generative policies from the mixed-quality data produced during deployment.

\paragraph{Offline RL and advantage-weighted policy improvement.}
Offline RL methods regularize policy improvement to avoid extrapolating beyond the dataset, using behavior constraints, conservative value estimation, or advantage-weighted updates~\citep{bcq,td3bc,cql,awr,awac}. Recent VLA-oriented RL systems often attach explicit value networks or critics to large generative policies~\citep{recap,calql,rl_vla,rl100,pirl,vlarl,digirl}. These critics can be effective but introduce additional optimization and inference cost. Flow-aware methods such as FQL~\citep{fql} adapt RL to flow policies through distillation. In contrast, \method~keeps the multi-step flow generator and embeds the quality signal as a generated coordinate, making the policy self-guided without a separate critic.

\paragraph{Self-guided generation.}
Best-of-$K$ sampling and guidance are widely used to improve generative-model outputs when a score is available. In robotics, such guidance is usually supplied by an external value function, reward model, or hand-designed cost. \method~learns the score jointly with the action endpoint, which is especially natural for flow policies: each sample carries both a candidate action and a potential forecast. The resulting inference rule is simple but effective, and the ablations show that its usefulness depends on calibrated potential learning.

\section{Conclusion}

We introduced \method, a self-guided flow-matching framework that improves VLA policies from mixed-quality robot experience by jointly generating action chunks and success potentials. Decoupled advantage-weighted flow matching preserves corrective supervision for failures while selectively improving actions, and a one-step CFM boundary estimator keeps advantage computation comparable in cost to supervised flow matching. Across BEHAVIOR-1K and real-world bimanual tasks, \method~improves over imitation baselines, matches a separate-critic baseline, and reduces training compute. Limitations remain: the method requires stage labels, sparse outcomes can still make very long-horizon credit assignment difficult, and the one-step estimator is approximate in finite networks. More capable robot policies may also introduce safety risks if deployed without validation, failure detection, and human oversight. Future work will investigate scaling self-guided flows to larger multi-task datasets, automatically inferred progress signals, and online settings where robots can reuse their own mixed-quality experience.

{
\clearpage
\small
\bibliographystyle{unsrtnat}
\bibliography{example_paper}
}

\newpage
\appendix

\hypersetup{linkcolor=black, citecolor=mydarkblue, urlcolor=black}

\renewcommand{\appendixpagename}{\centering \LARGE Appendix}
\appendixpage

\begingroup
\normalsize
\newcommand{\appendixdotfill}{\leaders\hbox to 0.5em{\hss.\hss}\hfill}
\newcommand{\appendixtocsection}[2]{\noindent\hyperref[#1]{\textbf{#2}}\appendixdotfill\pageref{#1}\par}
\newcommand{\appendixtocsubsection}[2]{\noindent\hspace*{2em}\hyperref[#1]{#2}\appendixdotfill\pageref{#1}\par}
\vspace{2ex}
\appendixtocsection{sec:theory}{A\quad Theoretical Analysis}
\appendixtocsubsection{app:theory_setup}{A.1\quad Setup and Assumptions}
\appendixtocsubsection{sec:theory_consistency}{A.2\quad Consistency of Success Potential Learning}
\appendixtocsubsection{app:theory_policy_improvement}{A.3\quad Policy Improvement via Weighted Flow Matching}
\appendixtocsubsection{sec:theory_hallucination}{A.4\quad Resolving Hallucination in Unified Architectures}
\appendixtocsubsection{app:mc_full}{A.5\quad MC Ranking Protocol}
\appendixtocsubsection{app:cross_task}{A.6\quad Cross-Task Within-Observation Validation}
\appendixtocsection{app:implementation_details}{B\quad Implementation Details}
\appendixtocsubsection{app:algorithm_pipeline}{B.1\quad \method~Algorithm Pipeline}
\appendixtocsubsection{app:model_config}{B.2\quad Model Configuration and Hyperparameters}
\appendixtocsubsection{app:il_hyper}{B.3\quad Imitation Learning Baselines}
\appendixtocsubsection{app:idql_impl}{B.4\quad IDQL Implementation}
\appendixtocsubsection{app:fql}{B.5\quad Flow Q-learning Implementation}
\appendixtocsection{app:hardware_components}{C\quad Hardware Components}
\appendixtocsection{app:task_definition}{D\quad Task Definition}
\appendixtocsubsection{app:simulation_tasks}{D.1\quad Simulation Tasks}
\appendixtocsubsection{app:real_world_tasks}{D.2\quad Real-World Robot Tasks}
\appendixtocsubsection{dataset_discription}{D.3\quad Dataset description}
\appendixtocsection{potential_example}{E\quad Qualitative Visualization}
\endgroup

\hypersetup{linkcolor=mydarkblue, citecolor=mydarkblue, urlcolor=mydarkblue}

\vspace{20ex}


\newpage

\section{Theoretical Analysis}
\label{sec:theory}


This appendix formalizes the two mathematical claims used in the main text. First, under independent endpoint sampling, the boundary velocity of the population CFM minimizer recovers the conditional data mean, whose potential component is the dataset-conditional success probability. Second, in the unclipped population setting, advantage-weighted flow matching on the action coordinates admits a reweighting interpretation under expected advantage weights. These results are population-level identities; the main paper therefore validates the finite-network approximations empirically.

\subsection{Setup and Assumptions}
\label{app:theory_setup}
Let $\pi_\beta(a|c)$ denote the behavior distribution of the dataset $\mathcal{D}$. Let $Y\in\{0,1\}$ denote the random binary chunk-level stage target and $y_{i,t}$ its observed realization. We define the action-quality function $Q(c,a)=\mathbb{E}[Y\mid c,a]$. Unless stated otherwise, all equalities involving $v^*$ refer to the population minimizer of the CFM objective.

\subsection{Consistency of Success Potential Learning}
\label{sec:theory_consistency}

We first justify using the augmented success potential vector $s$ as a baseline that recovers the dataset-conditional success probability $P_{\mathcal{D}}(Y=1 \mid c)$ at the population optimum. The connection to the standard MDP state value function $V^{\pi_\beta}$ is discussed in Remark~\ref{rem:value_interpretation}. In \method, the success potential dimension is trained via flow matching with uniform weights ($w=1$) on sparse binary targets $Y \in \{0, 1\}$.

\begin{theorem}[Boundary Identity for Conditional CFM]
\label{thm:consistency}
Suppose:
\begin{enumerate}
\item (Linear interpolation path) Endpoints $x_0 \sim p_0$ (the noise prior, e.g., $\mathcal{N}(0, I)$) and $x_1 \sim p_{\text{data}}(\cdot \mid c)$ (the data conditional on context $c$) are coupled via the linear interpolation $x_\sigma = (1 - \sigma)x_0 + \sigma x_1$ with target velocity $u_\sigma = x_1 - x_0$, $\sigma \in [0, 1]$.
\item (Independent endpoint sampling) The noise endpoint is sampled independently of the data pair: $x_0 \perp\!\!\!\perp (x_1, c)$. We do \emph{not} use minibatch optimal-transport coupling.
\item (Population minimizer with regularity) Let $v^*$ achieve the global minimum of the CFM squared regression loss
\[
\mathcal{L}_{\text{CFM}}(v) = \mathbb{E}_{\sigma \sim p(\sigma), \, x_0, x_1, c} \left[ \| v(x_\sigma, \sigma, c) - u_\sigma \|^2 \right].
\]
We assume $v^*$ admits a continuous extension as $\sigma \downarrow 0$ under standard regularity of the data distribution and $p(\sigma)$.
\item (Sparse undiscounted reward) Episode/stage outcomes are encoded as a binary label $Y \in \{0, 1\}$, and the success-potential endpoint is $x_1^{(s)}=Y\mathbf{1}_{H}$; reward is sparse and undiscounted ($\gamma = 1$).
\end{enumerate}
Then the one-step boundary estimator
\begin{equation}
\hat{x}_1 := x_0 + v^*(x_0, 0, c)
\end{equation}
satisfies, for $p_0$-almost every $x_0$,
\begin{equation}
\hat{x}_1 = \mathbb{E}_{x_1 \sim p_{\text{data}}(\cdot|c)}[x_1].
\end{equation}
In particular, the success-potential component recovers the dataset-conditional success probability after chunk-level averaging:
\begin{equation}
\hat{x}_1^{(s)}=\mathbb{E}[Y\mid c]\mathbf{1}_{H},\qquad
\hat{V}(c) := \tfrac{1}{H} \sum_{k=1}^{H} \hat{x}_{1,k}^{(s)} = \mathbb{E}[Y \mid c] = P_{\mathcal{D}}(Y = 1 \mid c).
\end{equation}
\end{theorem}

\begin{proof}
By the L2-Bayes regressor characterization, the population minimizer of the squared loss is the conditional mean:
\begin{equation}
v^*(x, \sigma, c)
=
\mathbb{E}\left[ x_1 - x_0 \mid x_\sigma = x, \sigma, c \right],
\end{equation}
defined $p_{\sigma}$-a.e. for each fixed $\sigma$. By Assumption~3, this extends continuously to $\sigma = 0$. At $\sigma = 0$, the flow state degenerates to pure noise: $x_{\sigma=0} = x_0$, so
\begin{equation}
 v^*(x_0, 0, c) = \mathbb{E}\left[ x_1 - x_0 \mid x_0, c \right] = \mathbb{E}[x_1 \mid x_0, c] - x_0.
\end{equation}
By Assumption~2, $x_0 \perp\!\!\!\perp (x_1, c)$, which by the graphoid weak-union axiom implies $x_0 \perp\!\!\!\perp x_1 \mid c$. Hence the conditional expectation simplifies to the marginal expectation:
\begin{equation}\mathbb{E}[x_1 \mid x_0, c] = \mathbb{E}_{x_1 \sim p_{\text{data}}(\cdot|c)} [x_1 \mid c]. \end{equation}
Substituting yields $v^*(x_0, 0, c) = \mathbb{E}[x_1 \mid c] - x_0$, so for $p_0$-a.e.\ $x_0$:
\begin{equation}\hat{x}_1 = x_0 + v^*(x_0, 0, c) = \mathbb{E}[x_1 \mid c]. \end{equation}
For the success-potential component, the target is $Y\mathbf{1}_{H}$, so
\begin{equation}\mathbb{E}[x_1^{(s)} \mid c]=\mathbb{E}[Y \mid c]\mathbf{1}_{H},\qquad \mathbb{E}[Y \mid c] = P_{\mathcal{D}}(Y=1 \mid c). \end{equation}
\end{proof}

\begin{remark}[Finite-capacity approximation]
\label{rem:finite_capacity}
Theorem~\ref{thm:consistency} concerns the population minimizer $v^*$ of the CFM objective. In practice, the trained network $v_\theta$ is a finite-capacity approximation to $v^*$, and we do not provide a formal bound on $\|v_\theta - v^*\|$. The theorem justifies the boundary estimate as a context baseline for AWR; it does not imply that the boundary value itself ranks multiple candidates under the same context. Section~\ref{sec:ablation} therefore treats the within-observation ranking induced by $v_\theta$ at $\mathrm{NFE}=1$ as an empirical finite-model property and compares it against a high-NFE reference.
\end{remark}

\begin{remark}[Single-step boundary approximation]
\label{rem:nfe1}
Theorem~\ref{thm:consistency} concerns the boundary value $v^*(x_0,0,c)$, whereas a high-NFE sampler follows the learned ODE and returns
\[
    x_1^{\mathrm{ODE}}
    =
    x_0+\int_0^1 v_\theta(x_\sigma^{\mathrm{ODE}},\sigma,c)\,d\sigma .
\]
The $\mathrm{NFE}=1$ estimator $\hat{x}_1=x_0+v_\theta(x_0,0,c)$ is therefore a boundary proxy obtained by evaluating the vector field only at $\sigma=0$, not an exact surrogate for the integrated ODE endpoint. Population optimality of the CFM objective alone does not make this proxy equal to $x_1^{\mathrm{ODE}}$; such equality would require additional structure not assumed here. Our use of $\mathrm{NFE}=1$ for ranking analysis relies on the weaker empirical property that the trained finite model's boundary proxy preserves within-observation rankings against a high-NFE reference, and the success-potential trajectories are near-linear in practice ($\kappa$ median $\le 0.039$ across tasks).
\end{remark}

\begin{remark}[State value function interpretation]
\label{rem:value_interpretation}
Under the additional assumption that the dataset $\mathcal{D}$ is generated by a single stationary behavior policy $\pi_\beta$ with sparse undiscounted reward and that $Y$ is the episode-level success outcome, $\hat{V}(c) = P_{\mathcal{D}}(Y = 1 \mid c)$ coincides with the standard MDP state value function $V^{\pi_\beta}(c)$. In our mixed-quality dataset (200 expert + 100 autonomous rollouts), $\pi_\beta$ is interpreted as the dataset-mixture policy.
\end{remark}

\subsection{Policy Improvement via Weighted Flow Matching}
\label{app:theory_policy_improvement}

\begin{theorem}[Population AWR Reweighting for Flow Matching]
\label{thm:policy_projection}
Fix a context $c$ and let $Q(c,a)=\mathbb{E}[Y\mid c,a]$ denote the population action-quality function. Consider the unclipped population weighted flow matching objective on the action coordinates with
\[
    \bar w(c,a)=\exp((Q(c,a)-V(c))/\tau),
\]
where $V(c)$ is independent of $a$. Then minimizing the weighted population CFM objective is equivalent, up to a positive context-dependent constant, to minimizing the CFM loss under the reweighted behavior distribution
\[
    \pi_{\bar w}(a|c)\propto \pi_\beta(a|c)\bar w(c,a)
    \propto \pi_\beta(a|c)\exp(Q(c,a)/\tau).
\]
\end{theorem}

\begin{proof}
Let $\mathcal{L}_{\text{CFM}}(\theta; c, a) = \mathbb{E}_{\sigma, x_0} [ \| v_\theta(x_\sigma, \sigma, c) - u_\sigma \|^2 ]$ denote the conditional flow matching loss for a specific action $a$.
The idealized population weighted objective for the policy is defined as:
\begin{equation} \mathcal{L}_\pi(\theta) = \mathbb{E}_{a \sim \pi_\beta(\cdot|c)} \left[ \bar w(c, a) \mathcal{L}_{\text{CFM}}(\theta; c, a) \right]. \end{equation}

We rewrite the expectation as an integral over the action space:
\begin{equation} \mathcal{L}_\pi(\theta) = \int \pi_\beta(a|c) \, \bar w(c, a) \, \mathcal{L}_{\text{CFM}}(\theta; c, a) \, da,
\label{24}
\end{equation}
Recall the closed-form solution for the optimal policy derived from the KL-regularized objective~\cite{awr}:
\begin{equation} \pi^*(a|c) = \frac{1}{Z(c)} \pi_\beta(a|c) \exp\left( \frac{Q(c, a)}{\tau} \right)\end{equation}
where $Z(c) = \int \pi_\beta(a|c) \exp\left( \frac{Q(c, a)}{\tau} \right) da$ is the partition function that normalizes the probability distribution.
By defining the population importance weight as $\bar w(c, a) = \exp( \frac{Q(c, a) - V(c)}{\tau} )$, we observe the following relationship between the distributions:
\begin{align} \pi_\beta(a|c) \bar w(c, a) &= \pi_\beta(a|c) \exp\left( \frac{Q(c, a)}{\tau} \right) \exp\left( \frac{-V(c)}{\tau} \right) \\ &= \left[ Z(c) \pi^*(a|c) \right] \exp\left( \frac{-V(c)}{\tau} \right). \end{align} 
Let $C(c) = Z(c) \exp(-V(c)/\tau)$ be the normalization constant which is independent of $a$. Substituting this back into Eq.~\ref{24}:

\begin{align} \mathcal{L}_\pi(\theta) &= \int \left[ C(c) \pi^*(a|c) \right] \mathcal{L}_{\text{CFM}}(\theta; c, a) \, da \\ &= C(c) \int \pi^*(a|c) \, \mathcal{L}_{\text{CFM}}(\theta; c, a) \, da \\ &= C(c) \, \mathbb{E}_{a \sim \pi^*(\cdot|c)} \left[ \mathcal{L}_{\text{CFM}}(\theta; c, a) \right]. \end{align} 

Since $C(c)$ is constant with respect to the network parameters $\theta$, minimizing the weighted objective $\mathcal{L}_\pi(\theta)$ is equivalent to minimizing the flow matching loss under the optimal policy distribution $\pi^*$:
\begin{equation} \nabla_\theta \mathcal{L}_\pi(\theta) \propto \nabla_\theta \mathbb{E}_{a \sim \pi^*} \left[ \mathcal{L}_{\text{CFM}}(\theta; c, a) \right]. \end{equation}
Thus, in the idealized unclipped population setting with expected advantage weights, weighted flow matching fits the action flow induced by the Boltzmann reweighted behavior distribution.
\end{proof}
\begin{remark}[Binary label implementation]
The implemented objective uses realized binary labels and clipped weights
$\min(M,\exp((y_{i,t}-\hat V_{i,t})/\tau))$. For stochastic binary labels, these weights are not an unbiased estimator of
$\exp((Q(c,a)-V(c))/\tau)$.
For the unclipped exact-baseline idealization, if
$Y\mid (c,a)\sim\mathrm{Bernoulli}(Q(c,a))$, then
\[
\mathbb{E}[\exp((Y-V(c))/\tau)\mid (c,a)]
=
\exp(-V(c)/\tau)\left(1+Q(c,a)(\exp(1/\tau)-1)\right).
\]
This quantity is monotone in $Q(c,a)$ but is not the exact Boltzmann factor. Therefore, the practical update should be interpreted as an AWR reweighting surrogate motivated by the population projection.
\end{remark}

\subsection{Resolving Hallucination in Unified Architectures}
\label{sec:theory_hallucination}

A key contribution of \method~is the decoupled AWR formulation. Here, we analyze why standard ``Coupled'' training, which applies the same AWR weight to the joint action-potential loss, can create a self-masking calibration failure in unified models.

Let $\mathcal{H}\subset\mathcal{D}$ be the set of hallucinated failure samples: trajectories with target label $y=0$ for which the current potential predicts high success, $\hat V_\theta\approx 1$. For these samples,
\begin{equation}
    A = y-\hat V_\theta \approx -1,
    \qquad
    w=\exp(A/\tau)\approx 0
\end{equation}
for small temperature $\tau$.

We compare the correction received by the success-potential predictor. Let $\theta_s$ denote parameters that affect the potential prediction; for shared parameters, the same argument applies to the potential-loss component of the shared gradient.

\paragraph{Case 1: Coupled AWR.}
The coupled objective applies $w$ to both the action loss and the potential loss. On $\mathcal{H}$, the potential-correction contribution is
\begin{equation}
    \nabla_{\theta_s}\mathcal{L}_{\text{coupled}}\big|_{\mathcal{H}}
    \supset
    w\,\nabla_{\theta_s}\mathcal{L}_{\text{potential}}
    \approx \mathbf{0}.
\end{equation}
The overestimated potential makes $A$ strongly negative, which makes $w$ nearly zero; the same small weight then masks the regression signal that would reduce $\hat V_\theta$. The action loss on these samples is also downweighted while $w\approx 0$, so the failure mode is not continued high-weight imitation of failure actions. Rather, the false-positive potential can remain uncorrected, corrupting later advantage estimates and self-guided candidate rankings.

\paragraph{Case 2: Decoupled AWR (Ours).}
The decoupled objective applies $w$ only to the action coordinates and keeps the potential regression uniformly weighted. Thus hallucinated failures still provide an unmasked potential correction:
\begin{equation}
    \nabla_{\theta_s}\mathcal{L}_{\text{decoupled}}\big|_{\mathcal{H}}
    \supset
    \nabla_{\theta_s}\mathcal{L}_{\text{potential}} .
\end{equation}
This gradient drives the predicted potential on failure-labelled samples toward their target, $\hat V_\theta\to y=0$. After this correction, the sample no longer receives a strongly negative advantage from value overestimation. Decoupling therefore breaks the self-masking loop: currently overestimated failures are downweighted in the action loss, while their potential targets remain visible for calibration, improving the reliability of both advantage weights and inference-time rankings.

\subsection{MC Ranking Protocol}
\label{app:mc_full}
For fairness across methods, both \method~and IDQL use the same MC ranking protocol at each $K$: $K$ candidate action chunks are independently sampled from the policy, and the chunk with the highest score is executed. \method's score is the chunk-level average of the embedded progress projection; IDQL's score is the chunk-level Q-value from the separate critic forward pass on each candidate.

\subsection{Cross-Task Within-Observation Validation}
\label{app:cross_task}

We extend the within-observation ranking analysis to all five BEHAVIOR-1K simulation tasks. Following a uniform protocol on the final \method~model, for each held-out observation we generate $K=8$ candidate actions, compute their progress scores under $\rm{NFE}=1$ (our estimator) and $\rm{NFE}=100$ (high-NFE reference), and report (i) Kendall $\tau$, (ii) top-$1$ agreement, (iii) pairwise accuracy, (iv) MC gain ratio, and (v) trajectory curvature $\kappa$ in the success-potential subspace. Aggregate per-task metrics are reported in Table~\ref{tab:nfe_cross_task}. Figure~\ref{fig:cross_task_scatter} visualizes score alignment between NFE$\,=\,$$1$ and NFE$\,=\,$$100$, Figure~\ref{fig:cross_task_tau} summarizes ranking fidelity across tasks and outcomes, and Figure~\ref{fig:cross_task_geom} analyzes trajectory geometry and value separation.

\begin{figure}[t]
    \centering
    \includegraphics[width=\linewidth]{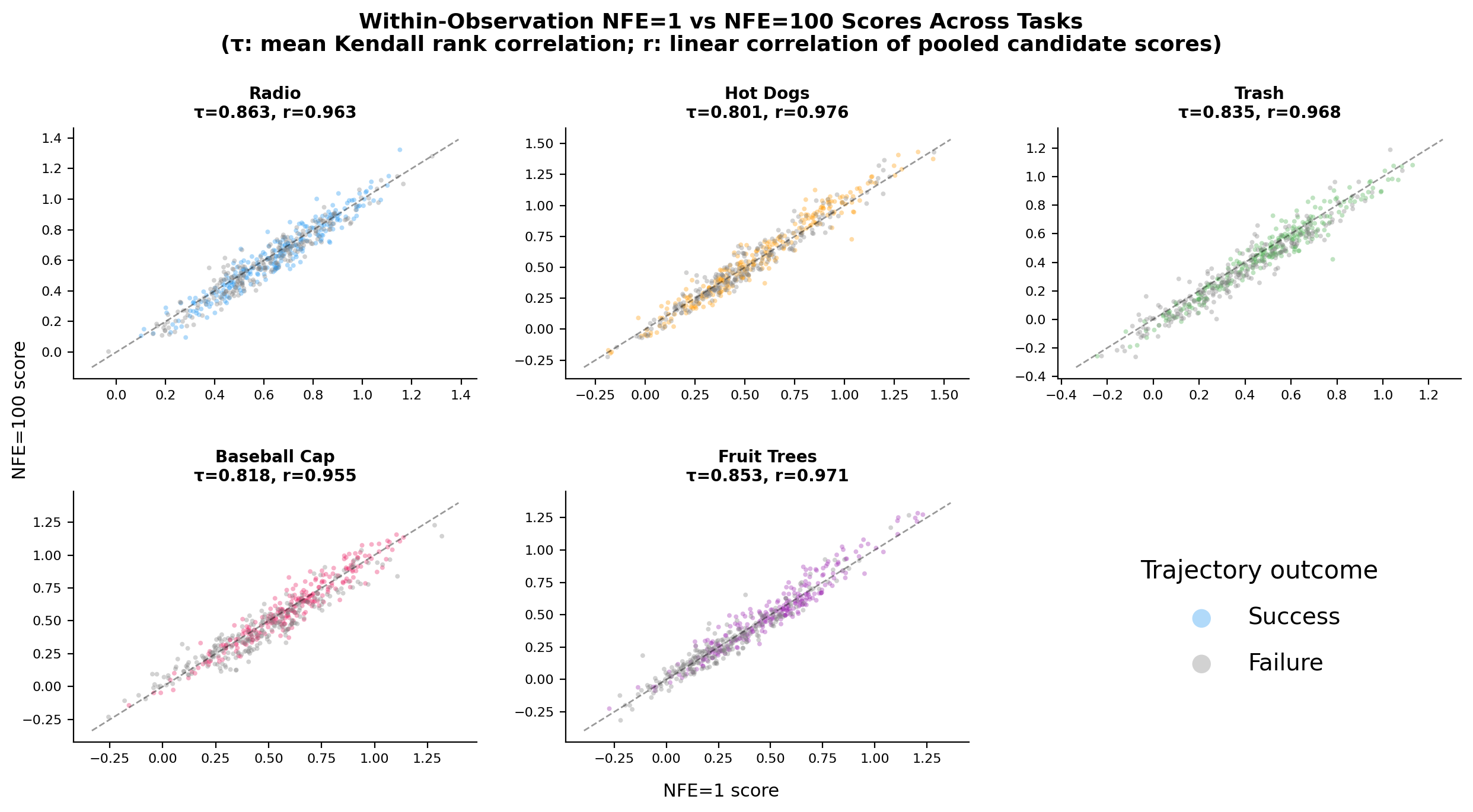}
    \caption{\textbf{NFE$=$1 vs NFE$=$100 score scatter (per task).} Each point is a candidate action's progress score under the two NFE settings. Monotone alignment along the diagonal indicates ranking preservation across all five tasks.}
    \label{fig:cross_task_scatter}
\end{figure}

\begin{figure}[t]
    \centering
    \begin{subfigure}{0.48\linewidth}
        \includegraphics[width=\linewidth]{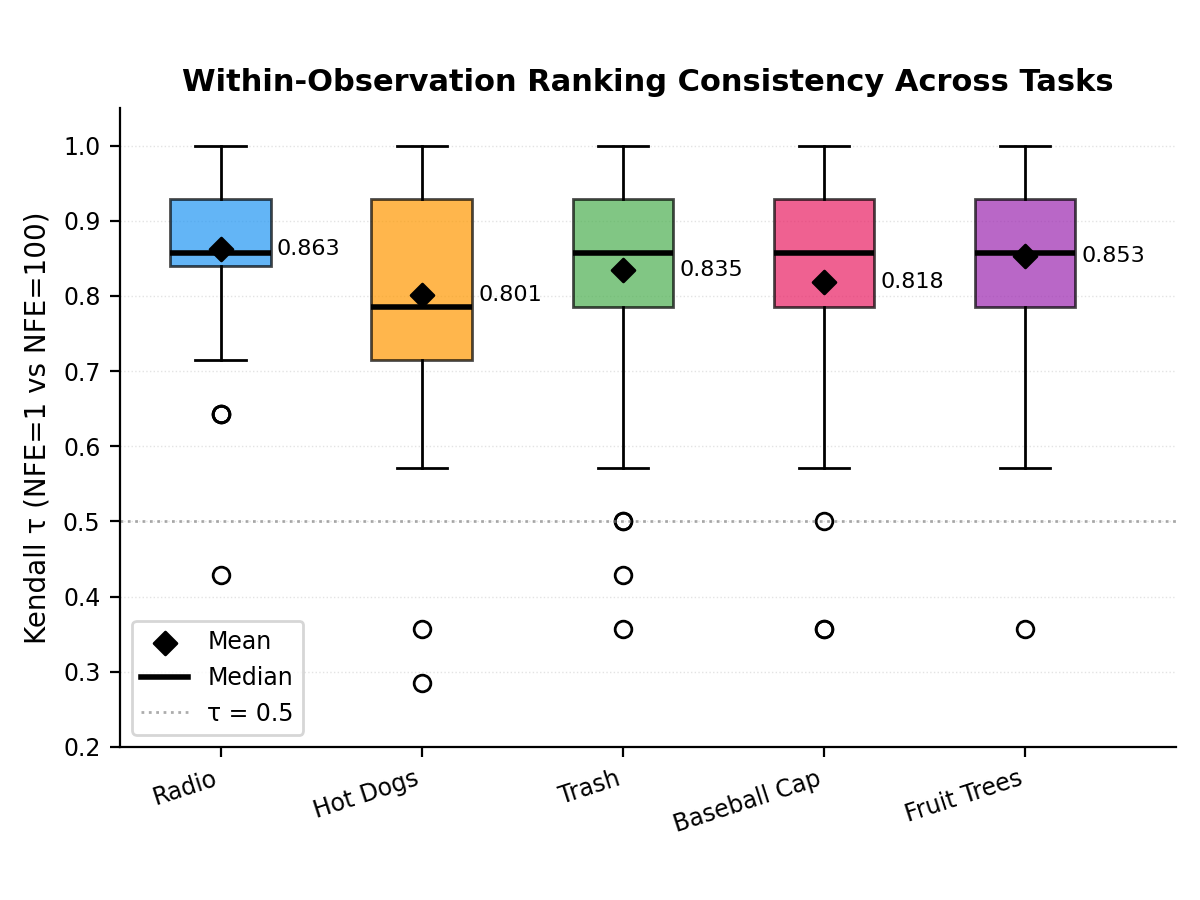}
        \caption{Kendall $\tau$ distribution per task.}
    \end{subfigure}
    \hfill
    \begin{subfigure}{0.48\linewidth}
        \includegraphics[width=\linewidth]{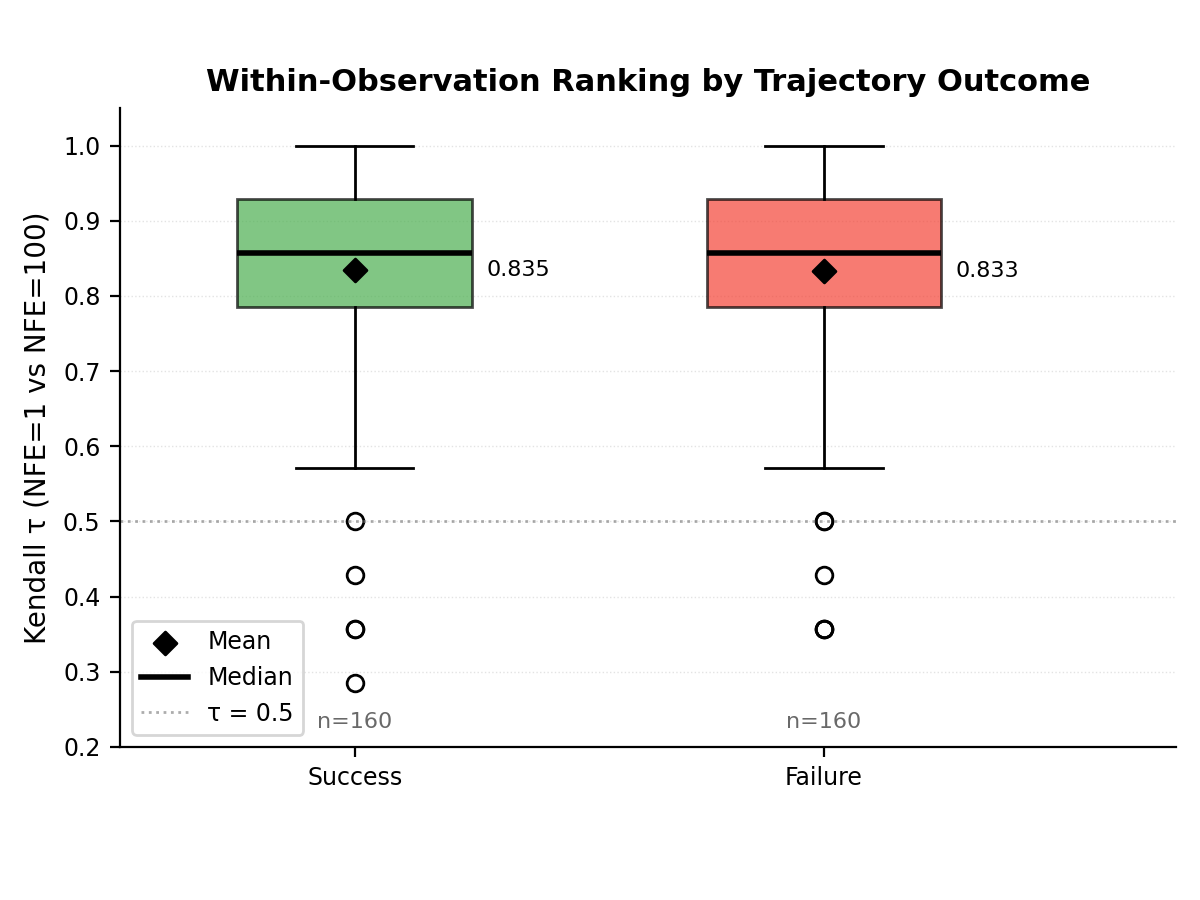}
        \caption{$\tau$ stratified by success vs failure trajectories.}
    \end{subfigure}
    \caption{\textbf{Per-task ranking fidelity.} (a) Per-task Kendall $\tau$ box plot. (b) $\tau$ remains high on both success and failure trajectories, addressing concerns about failure-mode error distribution.}
    \label{fig:cross_task_tau}
\end{figure}

\begin{figure}[t]
    \centering
    \begin{subfigure}{0.48\linewidth}
        \includegraphics[width=\linewidth]{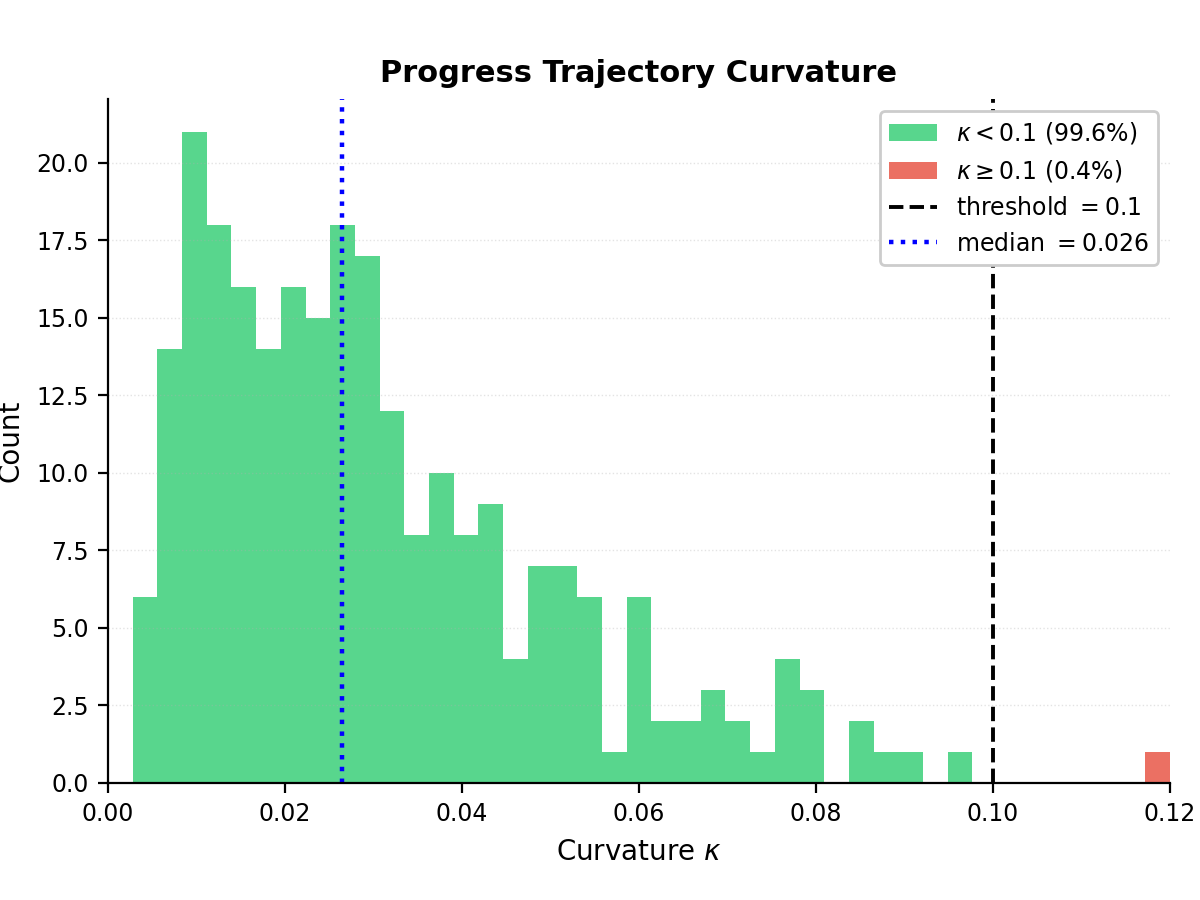}
        \caption{Trajectory curvature $\kappa$ distribution.}
    \end{subfigure}
    \hfill
    \begin{subfigure}{0.48\linewidth}
        \includegraphics[width=\linewidth]{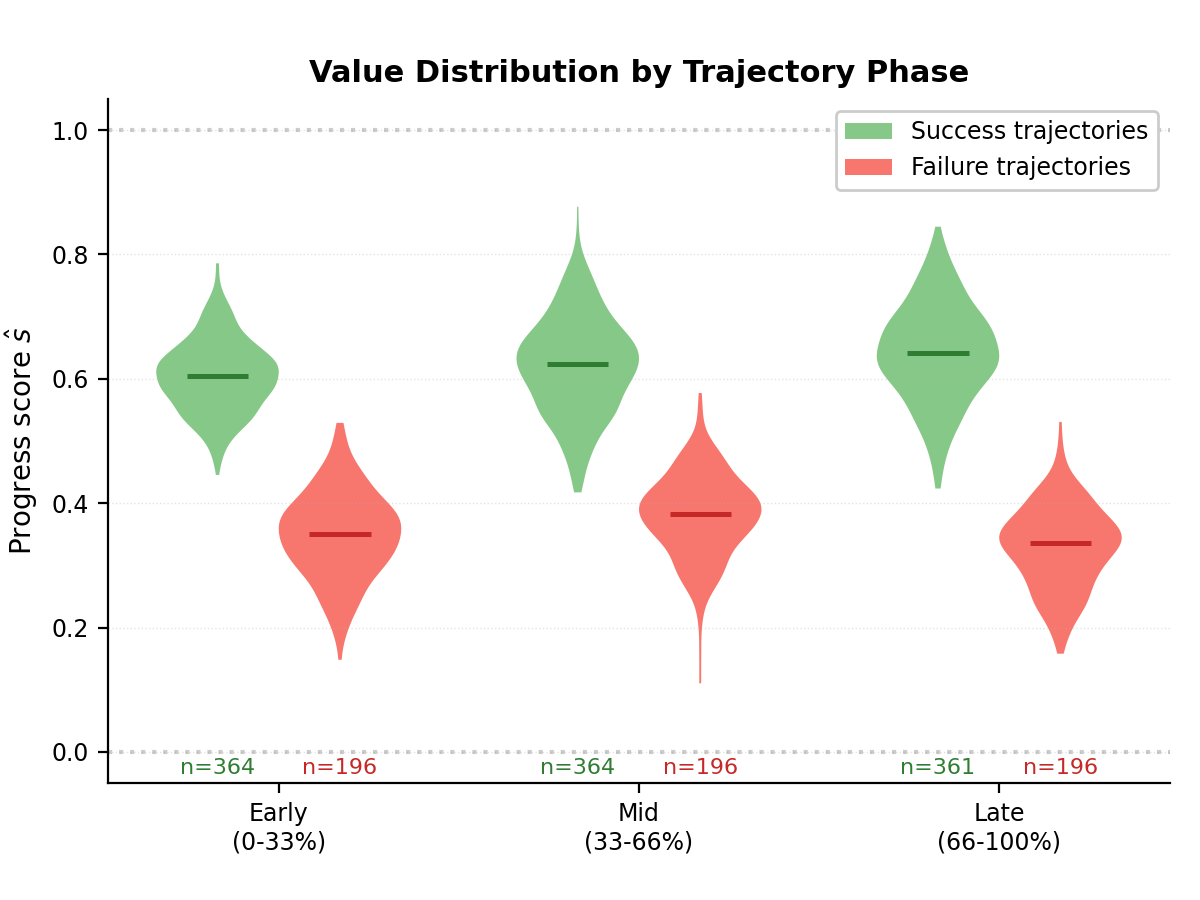}
        \caption{Per-stage value distribution.}
    \end{subfigure}
    \caption{\textbf{Trajectory geometry and value calibration.} (a) Curvature $\kappa$ in the success-potential subspace is concentrated near zero (median $\le 0.039$ across tasks), supporting the near-linearity assumption. (b) Per-stage value violin plot shows separation between success and failure trajectories.}
    \label{fig:cross_task_geom}
\end{figure}

\paragraph{Score scatter.}
Figure~\ref{fig:cross_task_scatter} plots $\rm{NFE}=1$ progress scores against the $\rm{NFE}=100$ reference for $K=8$ candidates per held-out observation, separately per task. The monotone alignment along the diagonal indicates that one-step scores closely track the high-NFE reference, supporting the ranking-fidelity results reported in Table~\ref{tab:nfe_cross_task}.

\paragraph{Per-task and outcome stratification.}
Figure~\ref{fig:cross_task_tau} (a) shows the Kendall $\tau$ distribution per task; (b) separates $\tau$ by success vs failure trajectories. Median $\tau \ge 0.80$ holds across both groups, addressing the concern that $\rm{NFE}=1$ error may be heavier on failure cases.

\paragraph{Trajectory geometry and value calibration.}
Figure~\ref{fig:cross_task_geom} (a) shows curvature $\kappa$ concentrated near zero (median $\le 0.039$), supporting the near-linearity assumption underlying single-step Euler integration. (b) Per-stage value violin plot shows separation between success and failure trajectories across stages, indicating that the embedded progress head provides a discriminative progress signal.


\newpage
\section{Implementation Details}
\label{app:implementation_details}
\subsection{\method~Algorithm Pipeline}
\label{app:algorithm_pipeline}

\method~is implemented as two decoupled procedures: an offline training algorithm and a test-time policy improvement algorithm.

\paragraph{Training.}
In training (Algorithm~\ref{alg:fpf_train}), we learn a conditional velocity field $v_\theta(x_\sigma,\sigma,c)$ over an augmented latent vector
$x=[a;s]\in\mathbb{R}^{H(d+1)}$, where $a\in\mathbb{R}^{dH}$ is an action chunk and $s\in[0,1]^H$ is a learned progress potential vector.
The potential is supervised by sparse stage-level quality via the broadcasted target $s^{\text{target}}_{i,t}$, and we perform policy improvement using
a Decoupled Advantage Weighted Regression objective: the AWR weight $w_{i,t}=\min(M,\exp(A_{i,t}/\tau))$ is applied only to the
action-component flow-matching loss, while the potential component is trained with uniform regression to remain a calibrated progress estimator.
To keep training efficient, the baseline is estimated with an $\mathrm{NFE}=1$ approximation
$\hat{x}_1 \approx x_0 + v_{\theta_{\text{old}}}(x_0, 0, c)$.

\begin{algorithm}[!h]
\caption{\method~Training: Decoupled Advantage Weighted Flow Matching}
\label{alg:fpf_train}
\small
\begin{algorithmic}[1]
\REQUIRE Offline dataset $\mathcal{D}=\{\xi_i\}_{i=1}^N$, trajectory $\xi_i=\{(\mathbf{o}_{i,t}, a_{i,t}, y_{i,t}, T_i)\}_{t=1}^{T_i}$;
chunk horizon $H$; flow time distribution $p(\sigma)$; noise prior $p_0(x)$; temperature $\tau$; clip constant $M$; learning rate $\eta$.
\ENSURE Trained velocity field $v_\theta$.

\vspace{2pt}
\STATE \textbf{Model:} augmented latent $x=[a; s]\in\mathbb{R}^{H(d+1)}$, with conditional velocity field $v_\theta(x_\sigma,\sigma,c)$.
\STATE Initialize parameters $\theta$.

\vspace{4pt}
\FOR{each gradient step}
    \STATE Sample trajectory index $i$ and timestep $t$ (ensure $t \le T_i-H+1$); set condition $c \leftarrow \mathbf{o}_{i,t}$.
    
    \STATE \textbf{Construct targets:}
    \STATE \hspace{6pt} Construct broadcast potential target: $s^{\text{target}}_{i,t} \leftarrow y_{i,t}\mathbf{1}_{H} \in \mathbb{R}^H$.

    \STATE Form endpoint (data) sample:
    \[
        x_1 \leftarrow \begin{bmatrix}a_{i,t:t+H-1}\\ s^{\text{target}}_{i,t}\end{bmatrix}.
    \]
    
    \STATE Sample $x_0\sim p_0(x)$ and $\sigma\sim p(\sigma)$; interpolate $x_\sigma\leftarrow (1-\sigma)x_0+\sigma x_1$.
    \STATE Flow-matching target velocity $u \leftarrow x_1-x_0$; split $u=(u^{(a)}, u^{(s)})$.

    \STATE \textbf{Efficient NFE=1 baseline estimation (Sec.~4.2):}
    \STATE \hspace{6pt} Predict $\hat{x}_1 \leftarrow x_0 + \sg(v_\theta(x_0, 0, c))$; extract $\hat{s}_{i,t} \leftarrow \hat{x}_1^{(s)}$.
    \STATE \hspace{6pt} Aggregate baseline: $\hat{V}_{i,t} \leftarrow \frac{1}{H} \sum_{k=1}^H \hat{s}_{i,t,k}$.

    \STATE Advantage $A_{i,t}\leftarrow y_{i,t} - \hat{V}_{i,t}$;
           weight $w_{i,t}\leftarrow \min\!\left(M,\exp\!\left(\frac{A_{i,t}}{\tau}\right)\right)$.

    \STATE Predict $v\leftarrow v_\theta(x_\sigma,\sigma,c)$; split $v=(v^{(a)}, v^{(s)})$.
    
    \STATE \textbf{Decoupled loss (Eq.~\ref{eq:decoupled_loss}):}
    \STATE \hspace{6pt} $\mathcal{L}(\theta)\leftarrow w_{i,t}\|v^{(a)}-u^{(a)}\|_2^2 + \|v^{(s)}-u^{(s)}\|_2^2$.
    \STATE Gradient update: $\theta \leftarrow \theta - \eta\nabla_\theta \mathcal{L}(\theta)$.
\ENDFOR
\end{algorithmic}
\end{algorithm}

\paragraph{Inference.}
At test time, we deploy the trained unified flow model to perform self-guided decision making, as detailed in Algorithm~\ref{alg:fpf_infer}. Unlike standard behavior cloning which greedily samples a single action, \method~leverages the learned potential field to assess the quality of its own generations. For a given context $c_t$, we generate $K$ parallel candidate trajectories by integrating the ODE from distinct noise samples. Since the backbone models the joint distribution $p_\theta(a, s | c)$, each generated sample automatically includes a dense progress potential vector $s^{(k)}_t$. We aggregate this vector into a scalar quality score $Q^{(k)}$ by averaging over the chunk horizon, serving as a representative metric for the trajectory's value. The policy then executes the action chunk corresponding to the highest score, using the learned potential for self-guided best-of-$K$ selection.

\begin{algorithm}[h]
\caption{\method~Inference: Monte Carlo Policy Improvement}
\label{alg:fpf_infer}
\small
\begin{algorithmic}[1]
\REQUIRE Trained velocity field $v_\theta$; noise prior $p_0(x)$; Monte Carlo samples $K$; current condition $c_t$.
\ENSURE Receding-horizon action chunk $a^*_t$.

\vspace{2pt}
\FOR{$k=1$ to $K$}
    \STATE Sample $x^{(k)}(0)\sim p_0(x)$.
    \STATE Integrate the probability flow ODE from $\sigma:0\rightarrow 1$:
    \[
        \frac{dx}{d\sigma}=v_\theta(x,\sigma,c_t),
    \]
    to obtain $x^{(k)}(1)=[a^{(k)}_t; s^{(k)}_t]$.
\ENDFOR
\STATE Compute chunk scores: $Q^{(k)} \leftarrow \frac{1}{H} \sum_{j=1}^H s^{(k)}_{t,j}$.
\STATE Select $k^* \leftarrow \arg\max_{k\in\{1,\dots,K\}} Q^{(k)}$.
\STATE Output $a^*_t \leftarrow a^{(k^*)}_t$ (execute in receding-horizon manner).
\end{algorithmic}
\end{algorithm}

\newpage
\subsection{Model Configuration and Hyperparameters}
\label{app:model_config}

We evaluate \method~on two distinct domains, employing architecture-specific modifications to integrate the potential field.

\paragraph{Simulation (BEHAVIOR-1K).}
We implement \method~on the \textbf{$\pi_{0.5}$} architecture~\citep{pi0.5}. To enable decoupled learning, we modify the Action Expert's output layer: the unified projection is replaced with two independent linear heads, an action head ($\mathbb{R}^{d_{\text{model}}} \to \mathbb{R}^d$) and a progress potential head ($\mathbb{R}^{d_{\text{model}}} \to \mathbb{R}^1$). Specific hyperparameters are listed in Table~\ref{tab:fpf_sim}.

\begin{table}[t]
\centering
\caption{Hyperparameters of \method~on simulation (based on $\pi_{0.5}$).}
\vtop{
\centering
\begin{tabular}{*{2}{c}}
\toprule
\textbf{Hyperparameter} & \textbf{Value}\\
\midrule
Action Expert Width & 1024\\
Action Expert Depth & 18\\
Action Expert MLP Dim & 4096\\
Batch Size & 128\\
Chunk Size & 30\\
Learning Rate & 1e-5\\
LR Scheduler & Cosine\\
Optimizer & AdamW \\
AdamW Betas & [0.9, 0.95]\\
AdamW Epsilon & 1e-8 \\
Weight Decay & 1e-10\\
Max Training Steps & 30,000 \\
Number of Flow Samples & 15\\
AWR Temperature $\tau$ & 0.3\\
AWR Clip Constant $M$ & 20.0\\
Fine-tune Method & Full parameter\\
\bottomrule
\end{tabular}
\label{tab:fpf_sim}
}
\end{table}

\paragraph{Real-World Robot.}
We implement \method~on the \textbf{$\pi_0$} architecture~\citep{pi0}. As $\pi_0$ predicts actions via learnable query tokens appended to the VLM sequence (without a separate Action Expert), we augment the token projection dimension to include the potential field and apply the same decoupled head strategy. Specific hyperparameters are listed in Table~\ref{tab:fpf_real}.

\begin{table}[t]
\centering
\caption{Hyperparameters of \method~on real-world robot (based on $\pi_0$).}
\vtop{
\centering
\begin{tabular}{*{2}{c}}
\toprule
\textbf{Hyperparameter} & \textbf{Value}\\
\midrule
Action Expert Width & 1024\\
Action Expert Depth & 18\\
Action Expert MLP Dim & 4096\\
Batch Size & 128\\
Chunk Size & 30\\
Learning Rate & 2.5e-5\\
LR Scheduler & Cosine\\
Optimizer & AdamW \\
AdamW Betas & [0.9, 0.95]\\
AdamW Epsilon & 1e-8 \\
Weight Decay & 1e-10\\
Max Training Steps & 20,000 \\
Fine-tune Method & Full parameter\\
Number of Flow Samples & 15\\
AWR Temperature $\tau$ & 0.3\\
AWR Clip Constant $M$ & 20.0\\
\bottomrule
\end{tabular}

\label{tab:fpf_real}
}
\end{table} 

\newpage
\subsection{Imitation Learning Baselines}
\label{app:il_hyper}

To ensure a rigorous comparison, our Behavior Cloning (BC) baselines use the exact same backbone architectures as \method: $\pi_{0.5}$~\citep{pi0.5} for simulation and $\pi_0$~\citep{pi0} for real-world tasks.

\paragraph{BC on Simulation (BEHAVIOR-1K).}
The simulation baseline utilizes the unmodified $\pi_{0.5}$ architecture. Unlike \method, this model retains the original unified output projection and is trained solely via the standard conditional flow matching objective on the mixed-quality dataset. It does not incorporate the success potential dimension or advantage weighting. Consequently, this baseline indiscriminately imitates both successful and failed trajectories, serving as a control to quantify the necessity of our decoupled policy improvement mechanism. Key hyperparameters are provided in Table~\ref{tab:bc_sim}.

\begin{table}[t]
\centering
\caption{Hyperparameters of BC on simulation (based on $\pi_{0.5}$).}
\vtop{
\centering
\begin{tabular}{*{2}{c}}
\toprule
\textbf{Hyperparameter} & \textbf{Value}\\
\midrule
Action Expert Width & 1024\\
Action Expert Depth & 18\\
Action Expert MLP Dim & 4096\\
Batch Size & 128\\
Chunk Size & 30\\
Learning Rate & 1e-5\\
LR Scheduler & Cosine\\
Optimizer & AdamW \\
AdamW Betas & [0.9, 0.95]\\
AdamW Epsilon & 1e-8 \\
Weight Decay & 1e-10\\
Max Training Steps & 30,000 \\
Fine-tune Method & Full parameter\\
\bottomrule
\end{tabular}

\label{tab:bc_sim}
}
\end{table}

\paragraph{BC on Real-World Robot.}
Similarly, the real-world baseline employs the standard $\pi_0$ architecture. It processes multi-view RGB observations and language instructions to directly generate action velocities via learnable tokens. As with the simulation baseline, it is trained using the standard flow matching loss without potential field augmentation. Key hyperparameters are listed in Table~\ref{tab:bc_real}.

\begin{table}[t]
\centering
\caption{Hyperparameters of BC on real-world robot (based on $\pi_0$).}
\vtop{
\centering
\begin{tabular}{*{2}{c}}
\toprule
\textbf{Hyperparameter} & \textbf{Value}\\
\midrule
Action Expert Width & 1024\\
Action Expert Depth & 18\\
Action Expert MLP Dim & 4096\\
Batch Size & 128\\
Chunk Size & 30\\
Learning Rate & 2.5e-5\\
LR Scheduler & Cosine\\
Optimizer & AdamW \\
AdamW Betas & [0.9, 0.95]\\
AdamW Epsilon & 1e-8 \\
Weight Decay & 1e-10\\
Max Training Steps & 20,000 \\
Fine-tune Method & Full parameter\\
\bottomrule
\end{tabular}
\label{tab:bc_real}
}
\end{table}

\subsection{IDQL Implementation}
\label{app:idql_impl}

We implement Implicit Diffusion Q-Learning (IDQL)~\citep{idql} by training a separate critic network to guide the policy.

\paragraph{Critic Architecture.}
The value model employs a Transformer backbone to naturally handle multi-modal inputs. As illustrated in Figure~\ref{fig:idql}, observations are processed via modality-specific encoders: three egocentric RGB views (front, left, right) are tokenized into visual sequences, while language instructions are embedded to provide goal guidance. Proprioceptive signals (e.g., joint angles, end-effector poses) are normalized and projected into a compact state vector. These heterogeneous tokens are concatenated and fed into a pretrained PaliGemma model~\citep{beyer2024paligemma}, using attention masks to manage cross-modal interactions. Finally, an MLP-based decoder maps the fused representation to scalar Q-value estimates.

\paragraph{Policy and Training.}
The policy architecture and hyperparameters are identical to the BC baselines described in Appendix~\ref{app:il_hyper}. The critic is trained via temporal difference learning, and the policy is fine-tuned using IQL-style advantage weighting. Detailed configurations are listed in Table~\ref{tab:idql-hyper}.

\begin{figure*}[htbp]
    \centering
    \includegraphics[width=0.85\linewidth]{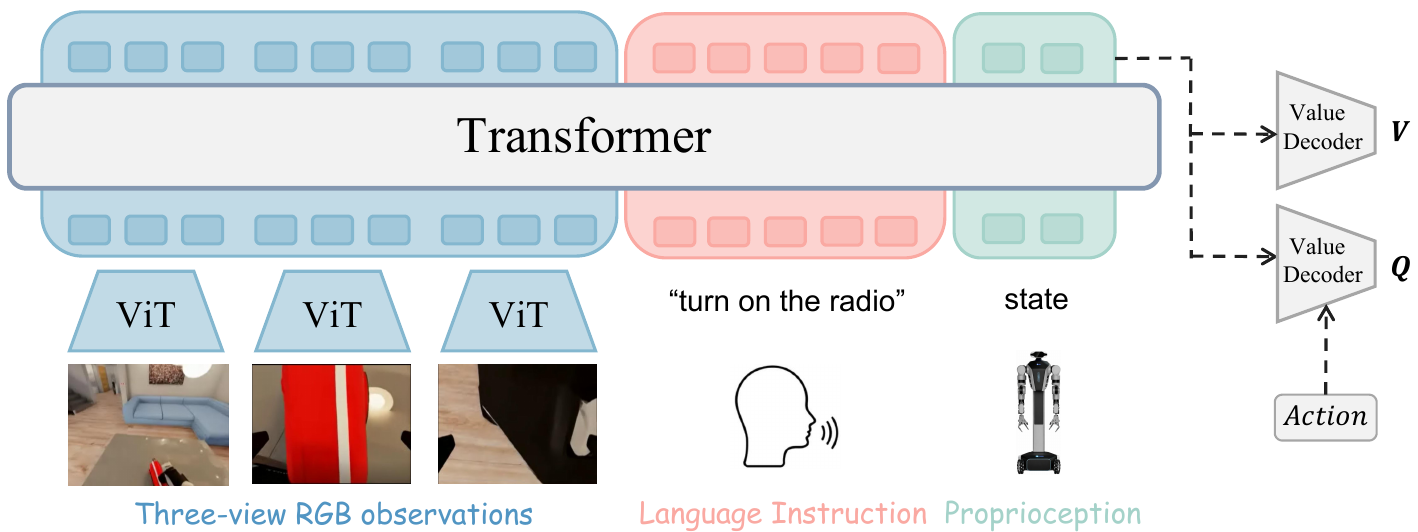}
    \caption{The architecture of the value model in Implicit Diffusion Q-Learning (IDQL).}
    \label{fig:idql}
    \vspace{-10pt}
\end{figure*}
 
\begin{table}[t]
\centering
\caption{Hyperparameters of value network of IDQL.}
\label{tab:idql-hyper}
\begin{tabular}{lc}
\toprule
\textbf{Hyperparameter} & \textbf{Value} \\
\midrule
\multicolumn{2}{l}{\textit{Value model Architecture}} \\
Value Network & MLP[$256$ $\times$ $1024$ $\times$ $256$]\\
Exponential Weight $\beta$ & $1$\\
Image Encoder & SigLIP~\citep{zhai2023sigmoid} \\
Transformer Model & PaliGemma~\citep{beyer2024paligemma} \\
Width & $256$\\
Depth & $4$ \\
MLP Dim & $1024$\\
Number of Heads & $8$\\
Number of KV Heads & $1$\\ 
\midrule
\multicolumn{2}{l}{\textit{Training Configuration}} \\
Batch Size & $64$ \\
Learning Rate & $1e^{-5}$ \\
LR Scheduler & Cosine \\
Optimizer & AdamW \\
AdamW Betas & $[0.9, 0.95]$ \\
Weight Decay & $1e^{-4}$ \\
Max Training Steps & $30,000$ \\
\bottomrule
\end{tabular}
\end{table}

\subsection{Flow Q-learning Implementation}
\label{app:fql}

We implement FQL~\citep{fql} on $\pi_{0.5}$ following the original method: a multi-step flow policy is first trained on the mixed-quality dataset, then distilled into a one-step student via Q-guided policy optimization. The Q function is trained jointly with the student policy on the same dataset. Detailed configurations are listed in Table~\ref{tab:fql-hyper}.

\paragraph{Model Architecture.}
We retain the pretrained vision-language prefix of $\pi_{0.5}$ and train only the action-side components used by the distilled student. For each observation and language instruction, the frozen VLM prefix encodes the visual and task context once, and this shared representation is consumed by both the multi-step flow teacher and the one-step student action expert. During Flow Q-learning training, the VLM prefix and the multi-step teacher action expert are kept frozen. The teacher therefore acts as a fixed reference policy that provides supervised flow targets, while the student action expert is the only policy component updated by the distillation, behavior-cloning and value-guided objectives. The critic is implemented as an independent twin-$Q$ architecture, where the two $Q$ networks use separate vision encoders, separate transformer backbones, and  separate action-value heads, enabling pessimistic value estimation through the minimum of the two predictions. 

\paragraph{Policy and Training.}
The $Q$ function is trained jointly with the student policy on the same mixed-quality offline dataset. At each update, the critic is fitted with a temporal-difference objective using dataset actions and target values computed from the target twin-$Q$ networks. The student is optimized to match the frozen flow teacher while also maximizing the learned critic value of its predicted actions. In addition to the original Flow Q-learning objective, we add a behavior-cloning anchor directly on the student action expert, which regularizes the student's predicted action toward the dataset action. Thus, the trainable student receives gradients from teacher distillation, critic-based policy improvement, and the BC anchor, while the VLM prefix and teacher remain frozen throughout training.

\begin{table}[h]
  \centering
  \caption{Hyperparameters of Flow Q-learning.}
  \label{tab:fql-hyper}
  \vspace{0.5em}
  \begin{tabular}{lc}
  \toprule
  \textbf{Hyperparameter} & \textbf{Value} \\
  \midrule
  \multicolumn{2}{l}{\textit{Q model architecture}} \\
  Vision encoder & SigLIP-M/$14$ ($\times 2$ separate) \\
  Transformer backbone & Gemma ($\times 2$ separate) \\
  Width & $256$ \\
  Depth & $4$ \\
  MLP dim & $1{,}024$ \\
  Attention heads / KV heads & $8$ / $1$ \\
  Action dim $\times$ Horizon & $32 \times 30$ \\
  Trainable SigLIP layers & $12$ \\
    \midrule
  \multicolumn{2}{l}{\textit{Student action expert architecture}} \\
    Action Expert Width & 1024\\
    Action Expert Depth & 18\\
    Action Expert MLP Dim & 4096\\
    Batch Size & 128\\
    Chunk Size & 30\\
  \midrule
  \multicolumn{2}{l}{\textit{Training}} \\
  $\gamma$ (discount factor) & $0.99$ \\
  Distillation weight $\alpha_{\mathrm{distill}}$ & $10$ \\
  BC anchor weight $\alpha_{\mathrm{BC}}$ & $5$ \\
  Teacher flow ODE steps & $20$ \\
  Target $Q$ EMA rate $\tau$ & $0.005$ \\
  Optimizer & AdamW \\
  Student learning rate & $1 \times 10^{-4}$ \\
  Q learning rate & $3 \times 10^{-5}$ \\
  Batch size & $128$ \\
  Max training steps & $20{,}000$ \\
  \bottomrule
  \end{tabular}
  \end{table}


\newpage

\section{Hardware Components}
\label{app:hardware_components}

\begin{figure*}[htbp]
    \centering
    \includegraphics[width=0.85\linewidth]{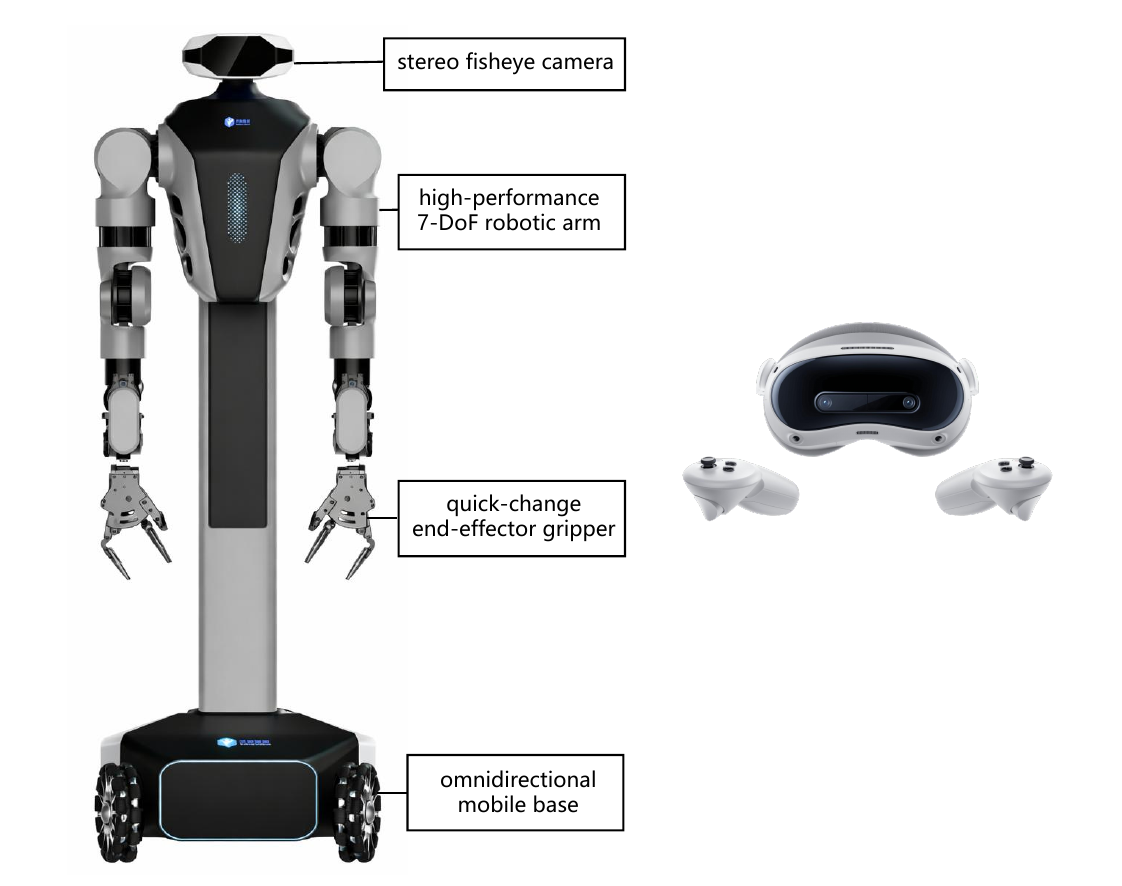}
    \caption{Hardware overview of the TeleAvatar robot and its teleoperation system.}
    \label{fig:hardware}
\end{figure*}
As shown in Figure~\ref{fig:hardware}, our hardware setup consists of a DexTeleop TeleAvatar Lite robot and a virtual reality (VR) teleoperation kit. The TeleAvatar Lite is equipped with two 7-DoF arms, each fitted with a DexTeleop custom parallel gripper.
For data collection, the robot carries a binocular fisheye camera module (two lenses) providing a 180$^\circ$ field of view and $3840 \times 1080$ video, with a glass-to-glass latency below 150 ms. In addition, two RGB-D cameras are mounted on the wrists to capture $1280 \times 720$ video. All three cameras operate at 30 Hz.
The VR kit comprises a head-mounted display, low-latency tracking sensors, and handheld controllers. The controller buttons are used to start, pause, save, and discard recordings. The logged modalities include VR-perspective RGB images, head-pose estimates, eye-tracking signals, interaction event logs, and motion control commands.


\section{Task Definition}
\label{app:task_definition}

\subsection{Simulation Tasks}
\label{app:simulation_tasks}

\paragraph{Turning on radio}

\paragraph{Task Description}
In this task, the robot must visually locate a target radio, navigate to it, grasp it with the right gripper, press the power switch with the left gripper, and then place the radio back on the table. A key challenge is combining mobile manipulation with bimanual control, since the robot must maintain a stable grasp while counteracting the contact forces induced by pressing the switch. In each trial, the radio is placed at a different position and orientation on the table within the robot's field of view and reachable workspace. The maximum episode horizon at inference time is 6000 steps.

\begin{figure*}[h]
    \centering
    \includegraphics[width=\linewidth]{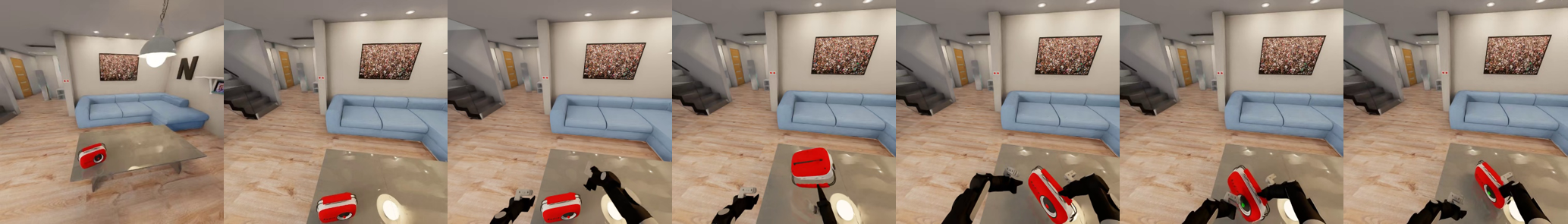}
    \caption{Turning on radio}
    \label{fig:Turning on radio}
\end{figure*}

\paragraph{Picking up trash}

\paragraph{Task Description}
In this task, the robot grasps a trash bin with its right gripper and carries it while navigating sequentially to three soda cans. At each stop, it picks up a can with the left gripper, drops it into the bin, and, after collecting all three cans, places the bin back on the floor. The main difficulty is the long-horizon mobile manipulation sequence, which requires asymmetric bimanual coordination: the right arm must stabilize the bin while the left arm performs repeated pick-and-place actions. In each trial, the trash bin and soda cans are placed at different locations, requiring active navigation. To further evaluate generalization, the robot's initial base pose is also randomized within a constrained range. The maximum episode horizon at inference time is 13,000 steps.

\begin{figure*}[h]
    \centering
    \includegraphics[width=\linewidth]{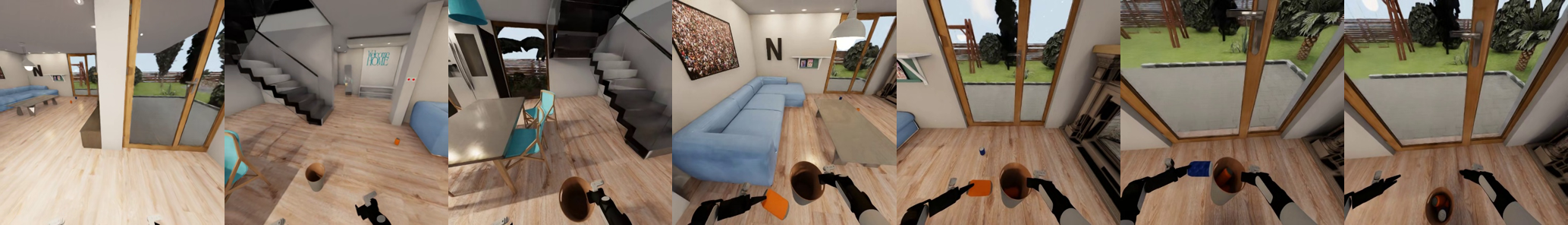}
    \caption{Picking up trash}
    \label{fig:Picking up trash}
\end{figure*}

\paragraph{Spraying fruit trees}

\paragraph{Task Description}
In this task, the robot must visually locate and grasp a watering device with its left gripper and then navigate sequentially to two target trees. At each tree, it uses the right gripper to switch the device to the ``on'' state, waters the tree, and switches the device back to ``off'' before moving to the next target. This setting is challenging because it combines mobile manipulation with multi-stage bimanual tool use: the left arm must stabilize the watering device and maintain its aim, while the right arm actuates the switch without disturbing the tool pose. In each trial, the watering device and the trees are randomly positioned within the workspace, requiring active visual search and planning. To further evaluate generalization, the robot's initial base pose is also varied. The maximum episode horizon at inference time is 10,000 steps.

\begin{figure*}[h]
    \centering
    \includegraphics[width=\linewidth]{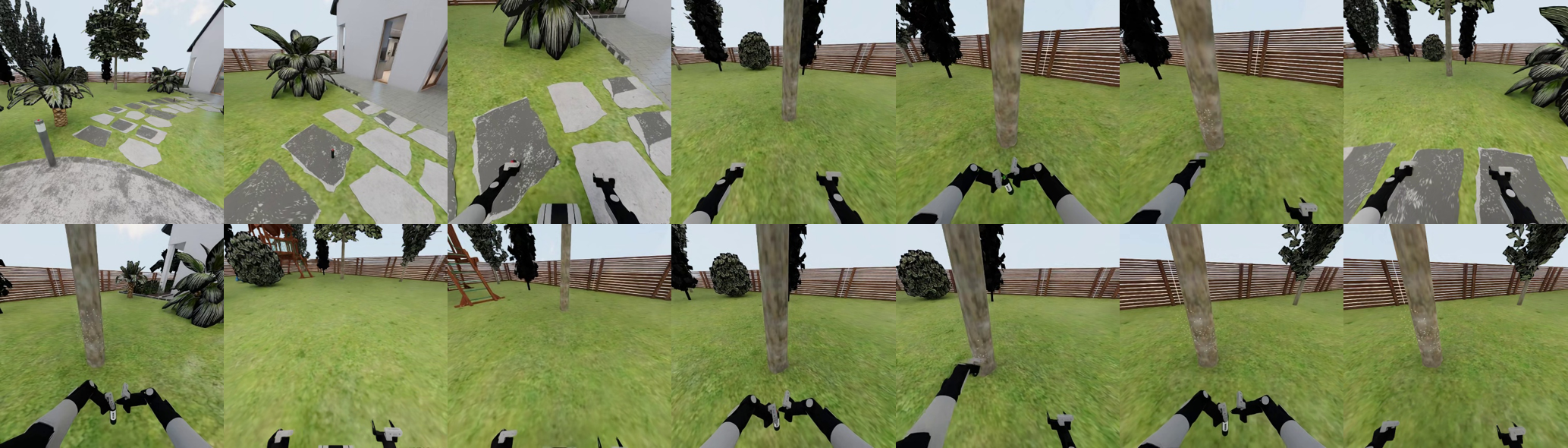}
    \caption{Spraying fruit trees}
    \label{fig:Spraying fruit trees}
\end{figure*}

\paragraph{Cook hot dogs}

\paragraph{Task Description}
In this task, the robot opens a refrigerator to retrieve two hot dogs, closes the door, places one hot dog on the tabletop, and puts the other into a microwave. It then transfers the staged hot dog into the microwave, closes the microwave door, and presses the start switch. A key challenge is long-horizon planning across multiple functional zones, requiring the robot to coordinate articulated-door manipulation with pick-and-place actions in the storage, staging, and cooking areas. In each trial, the initial object positions and the robot's base pose are varied to evaluate generalization. The maximum episode horizon at inference time is 15,000 steps.

\begin{figure*}[h]
    \centering
    \includegraphics[width=\linewidth]{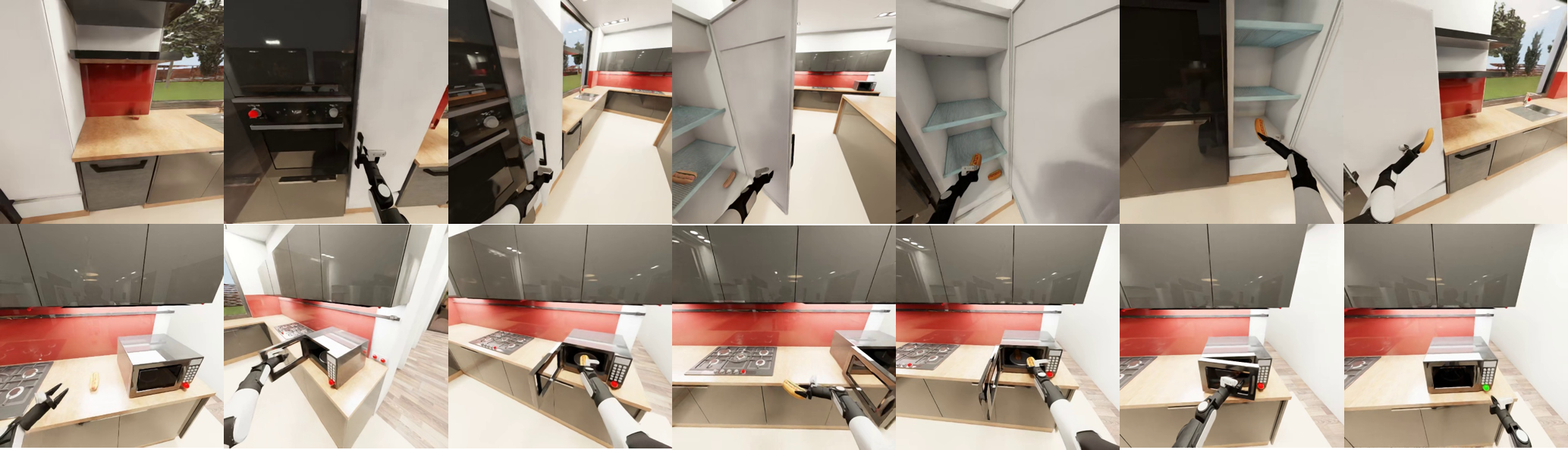}
    \caption{Cook hot dogs}
    \label{fig:Cook hot dogs}
\end{figure*}

\newpage

\paragraph{Wash a baseball cap}

\paragraph{Task Description}
In this task, the robot navigates to a washing machine, opens its door, and finds two baseball caps. It grasps one cap with each gripper, carries them to the washer, places both caps into the confined drum one after the other, closes the door, and presses the power switch. The main challenge is long-horizon mobile manipulation that combines articulated object control with simultaneous bimanual carrying. The robot must navigate with both grippers occupied and place the caps accurately into the confined drum. In each trial, the initial cap positions and the robot's starting base pose are varied to evaluate generalization. The maximum episode horizon at inference time is 10,000 steps.

\begin{figure*}[h]
    \centering
    \includegraphics[width=\linewidth]{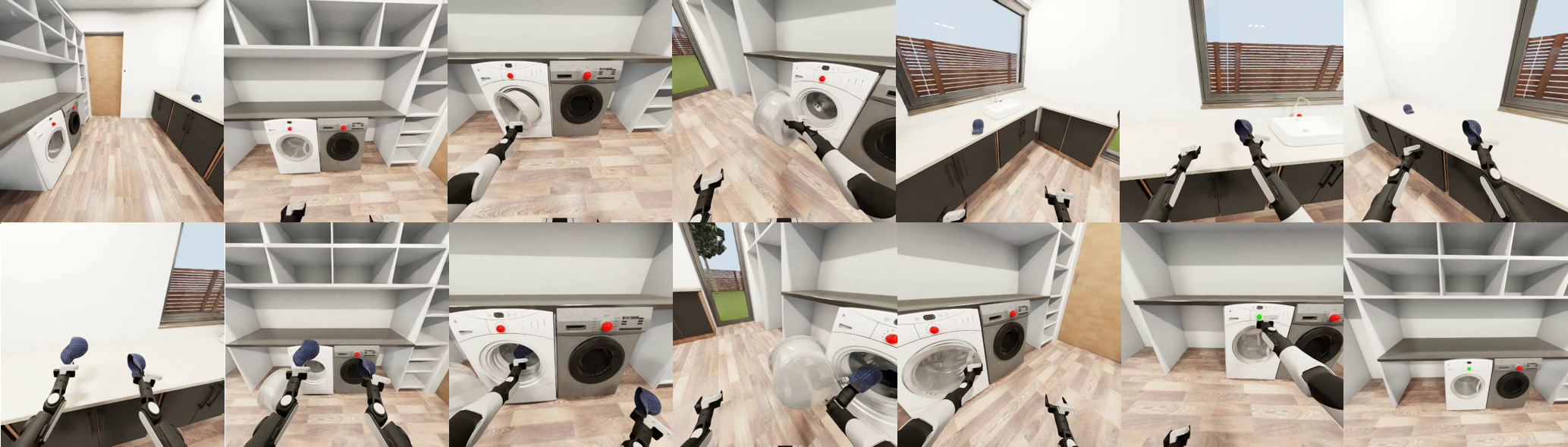}
    \caption{Wash a baseball cap}
    \label{fig:Wash a baseball cap}
\end{figure*}

\subsection{Real-World Robot Tasks}
\label{app:real_world_tasks}


\paragraph{Set-Paper-Roll}
\label{ssec:{pick_paper}}

\paragraph{Task Description} 
In this task, the robot picks up a toilet paper roll with its right gripper, transfers it to the left gripper, and places it onto a cylindrical holder. A key challenge is the multi-stage sequence itself, which couples grasp acquisition, coordinated bimanual handover, and accurate placement onto the holder. In each trial, the toilet paper roll and the cylindrical holder are randomly positioned on the tabletop within the robot's workspace and field of view. To evaluate generalization, the robot's initial base pose is also randomized within a constrained range.

\begin{figure*}[h]
    \centering
    \includegraphics[width=\linewidth]{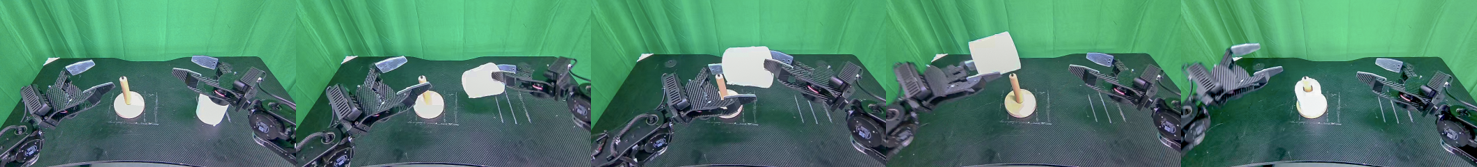}
    \caption{Set-Paper-Roll}
    \label{fig:Set-Paper-Roll}
\end{figure*}

\paragraph{Stage 1: Grasp paper roll}

\begin{itemize}
    \item \texttt{0.0 points}: The robot failed to grasp the toilet paper roll.
    \item \texttt{0.5 points}: The robot grasped the roll but dropped it afterward.
    \item \texttt{1.0 points}: The robot successfully grasped the roll with the right gripper.
\end{itemize}

\paragraph{Stage 2: Transport paper roll}

\begin{itemize}
    \item \texttt{0.0 points}: The robot failed to transfer the roll from the right gripper to the left gripper.
    \item \texttt{0.5 points}: The robot dropped the roll during the handover.
    \item \texttt{1.0 points}: The robot successfully transferred the roll from the right gripper to the left gripper.
\end{itemize}

\paragraph{Stage 3: Place paper roll onto holder}

\begin{itemize}
    \item \texttt{0.0 points}: The robot failed to place the roll onto the holder.
    \item \texttt{0.5 points}: The robot placed the roll on the holder but did not fully insert it.
    \item \texttt{1.0 points}: The robot successfully inserted the roll onto the holder.
\end{itemize}

\paragraph{Pick-Trash}
\label{ssec:{put_trash}}

\paragraph{Task Description}  
In this task, the robot opens the lid of a tabletop trash bin with its left gripper, repeatedly picks up crumpled paper balls from the tabletop with its right gripper, drops them into the bin, and finally closes the lid. The main challenge is twofold: the paper balls are randomly distributed and deformable, and the lid-opening motion depends on the bin geometry and contact conditions. In each trial, the trash bin, the paper balls, and the robot's initial base and torso poses are randomized within a constrained range.

\begin{figure*}[h]
    \centering
    \includegraphics[width=\linewidth]{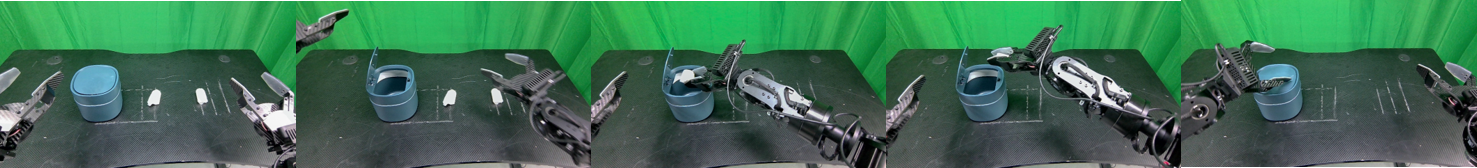}
    \caption{Pick-Trash }
    \label{fig:Pick-Trash (Real)}
\end{figure*}

\paragraph{Stage 1: Open trash bin lid}

\begin{itemize}
    \item \texttt{0.0 points}: The robot failed to open the trash bin lid.
    \item \texttt{0.5 points}: The robot opened the trash bin lid, but only after multiple attempts.
    \item \texttt{1.0 points}: The robot successfully opened the trash bin lid on the first attempt.
\end{itemize}

\paragraph{Stage 2: Place paper balls in trash bin}

\begin{itemize}
    \item \texttt{0.0 points}: The robot failed to place any paper ball into the trash bin.
    \item \texttt{0.5 points}: The robot placed one paper ball into the trash bin.
    \item \texttt{1.0 points}: The robot successfully placed both paper balls into the trash bin.
\end{itemize}

\paragraph{Stage 3: Close trash bin lid}

\begin{itemize}
    \item \texttt{0.0 points}: The robot failed to close the trash bin lid.
    \item \texttt{0.5 points}: The robot closed the trash bin lid, but only after multiple attempts.
    \item \texttt{1.0 points}: The robot successfully closed the trash bin lid on the first attempt.
\end{itemize}

\paragraph{Cube-Stack}
\label{ssec:{build_blocks}}
\paragraph{Task Description}
In this task, the robot sequentially grasps three colored blocks from the tabletop and stacks them on a plate while ensuring that no block falls off during placement. The robot selects the manipulation arm based on each block's position relative to the plate: blocks on the left are grasped with the left gripper, whereas blocks on the right are grasped with the right gripper. A key challenge is controlling the end-effector pose during placement and stacking so as to maintain stability and avoid slippage or toppling within the confined area of the plate. In each trial, the blocks and plate, as well as the robot's initial base and torso poses, are randomized within a constrained range.

\begin{figure*}[h]
    \centering
    \includegraphics[width=\linewidth]{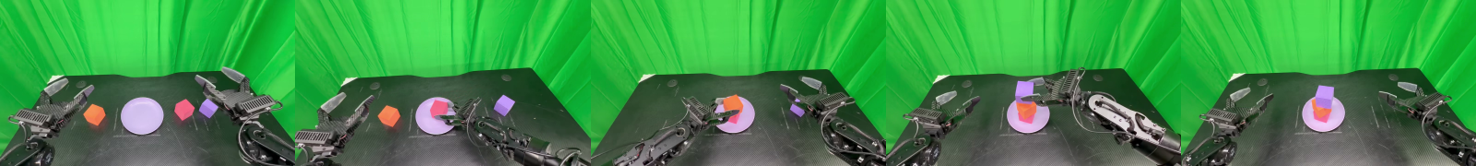}
    \caption{Cube-Stack}
    \label{fig:cube-stack}
\end{figure*}

\paragraph{Stage 1: Place red cube on plate}

\begin{itemize}
    \item \texttt{0.0 points}: The robot failed to grasp the red cube.
    \item \texttt{0.5 points}: The robot grasped the red cube but failed to place it on the plate.
    \item \texttt{1.0 points}: The robot successfully placed the red cube on the plate.
\end{itemize}

\paragraph{Stage 2: Stack orange cube on red cube}

\begin{itemize}
    \item \texttt{0.0 points}: The robot failed to grasp the orange cube.
    \item \texttt{0.5 points}: The robot grasped the orange cube but failed to stack it on the red cube.
    \item \texttt{1.0 points}: The robot successfully stacked the orange cube on the red cube.
\end{itemize}

\paragraph{Stage 3: Stack purple cube on orange cube}

\begin{itemize}
    \item \texttt{0.0 points}: The robot failed to grasp the purple cube.
    \item \texttt{0.5 points}: The robot grasped the purple cube but failed to stack it on the orange cube.
    \item \texttt{1.0 points}: The robot successfully stacked the purple cube on the orange cube.
\end{itemize}

\paragraph{Transfer-Food}
\label{ssec:{Transfer-Food}}
\paragraph{Task Description}
In this task, the robot first picks several peanuts from the table with its right gripper and places them into a small bowl on the right that already contains some peanuts. It then grasps the small bowl, pours all of its contents into a larger central bowl, and returns the small bowl to its original position. Finally, it uses the left gripper to pick up assorted bread items and place them into the large bowl. This task is challenging because the bread items are irregularly shaped and require different grasping strategies, while the overall sequence also involves small-object manipulation and a pouring action. To evaluate generalization, the quantities, types, and initial poses of the food items, as well as the robot's base position, are randomized within the workspace in every trial.

\begin{figure*}[h]
    \centering
    \includegraphics[width=\linewidth]{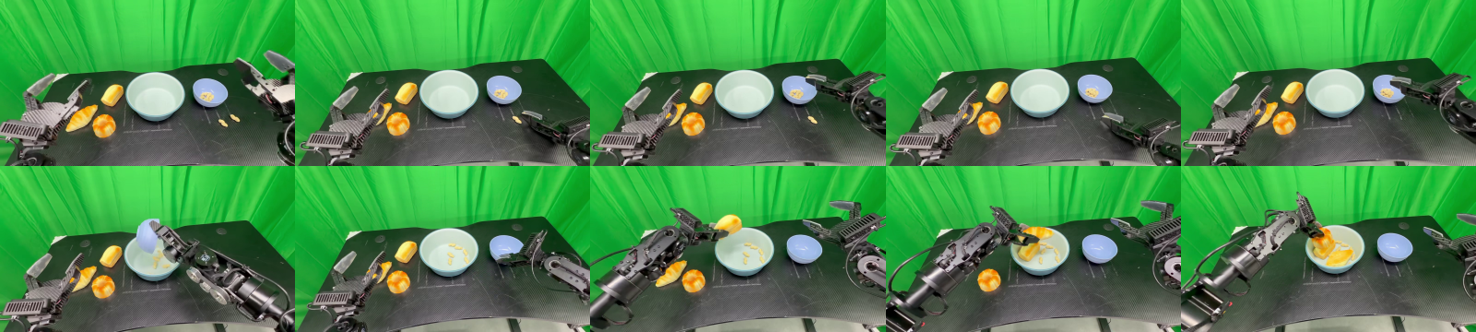}
    \caption{Transfer-Food}
    \label{fig:Transfer-Food}
\end{figure*}

\paragraph{Stage 1: Peanut Pick-and-Place}

\begin{itemize}
    \item \texttt{0.0 points}: The robot failed to pick and place any peanut.
    \item \texttt{0.5 points}: The robot placed the first peanut into the bowl but failed to place the second peanut.
    \item \texttt{1.0 points}: The robot successfully placed both peanuts into the bowl.
\end{itemize}

\paragraph{Stage 2: Peanut Pouring}

\begin{itemize}
    \item \texttt{0.0 points}: The robot failed to grasp the bowl.
    \item \texttt{0.5 points}: The robot grasped the bowl but failed to complete the pouring.
    \item \texttt{1.0 points}: The robot successfully poured all peanuts into the large bowl.
\end{itemize}

\paragraph{Stage 3: Bread Transfer}

\begin{itemize}
    \item \texttt{0.0 points}: The robot failed to transfer any bread into the large bowl.
    \item \texttt{0.5 points}: The robot deposited the first two pieces of bread but failed to transfer the third.
    \item \texttt{1.0 points}: The robot successfully transferred all three pieces of bread into the large bowl.
\end{itemize}

\paragraph{Wipe-Whiteboard}
\label{ssec:{Wipe-Whiteboard}}
\paragraph{Task Description}
In this task, the robot lifts a cover plate with its left gripper to reveal hidden text on a whiteboard, grasps the eraser with its right gripper, wipes the board clean, and finally returns the plate to its original position. The main difficulty is the long-horizon sequence that combines coordinated bimanual manipulation with surface interaction. Unlike standard pick-and-place actions, the erasing step requires force modulation and trajectory tracking to remove the writing effectively. To evaluate generalization, the initial poses of the plate, eraser, and text, as well as the robot's base position, are randomized within the workspace in every trial.

\begin{figure*}[h]
    \centering
    \includegraphics[width=\linewidth]{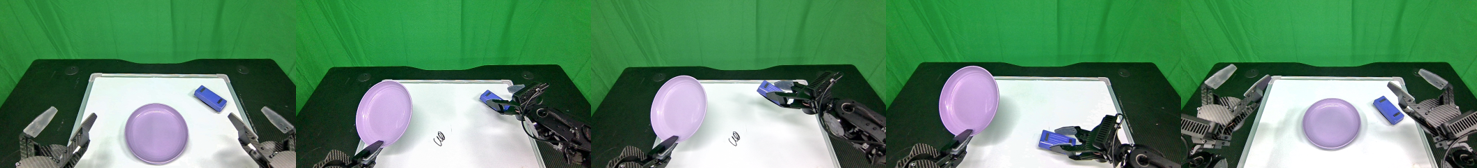}
    \caption{Wipe-Whiteboard}
    \label{fig:Wipe-Whiteboard}
\end{figure*}

\paragraph{Stage 1: Grasp the plate}

\begin{itemize}
    \item \texttt{0.0 points}: The robot failed to grasp the plate.
    \item \texttt{0.5 points}: The plate fell after being grasped.
    \item \texttt{1.0 points}: The robot successfully grasped the plate.
\end{itemize}

\paragraph{Stage 2: Grasp the eraser}

\begin{itemize}
    \item \texttt{0.0 points}: The robot failed to grasp the eraser.
    \item \texttt{0.5 points}: The robot grasped the eraser but with a pose unsuitable for erasing.
    \item \texttt{1.0 points}: The robot grasped the eraser in a pose suitable for erasing.
\end{itemize}

\paragraph{Stage 3: Wipe markings}
\begin{itemize}
    \item \texttt{0.0 points}: The robot did not execute the erasing motion.
    \item \texttt{0.5 points}: The robot executed the erasing motion but failed to remove the writing.
    \item \texttt{1.0 points}: The robot successfully removed the writing from the whiteboard.
\end{itemize}

\paragraph{Stage 4: Return the plate}

\begin{itemize}
    \item \texttt{0.0 points}: The plate was dropped during transport, or no placement was attempted.
    \item \texttt{0.5 points}: The robot attempted to place the plate but did not return it to its original position.
    \item \texttt{1.0 points}: The robot successfully placed the plate at its original position on the whiteboard.
\end{itemize}

\subsection{Dataset description}
\label{dataset_discription}
For each task in the simulation and real-world datasets, we construct a mixed-quality dataset consisting of 200 expert demonstrations and 100 autonomous rollouts.
For the simulation tasks, we use official human demonstrations from the BEHAVIOR-1K Challenge, while the real-world trajectories are collected via VR teleoperation.
To increase behavioral diversity and capture failure modes, we additionally generate autonomous rollouts using a pretrained $\pi_{0.5}$ checkpoint~\cite{pi0.5}. These rollouts span a broad range of outcomes, including complete successes, partial completions, and failures.
Figure~\ref{fig:data_dist} shows the frame distribution across the 10 tasks in the simulation and real-world datasets, together with the numbers of successful and failed trajectories collected from autonomous rollouts. In addition, all trajectories are annotated with sparse stage-wise labels, from which we derive the binary targets $y$ defined in Section~\ref{sec:method}. For IDQL, we assign an additional label of $-1$ to the final step of each failed rollout episode.

\begin{figure}[t]
    \centering
    \includegraphics[width=\linewidth]{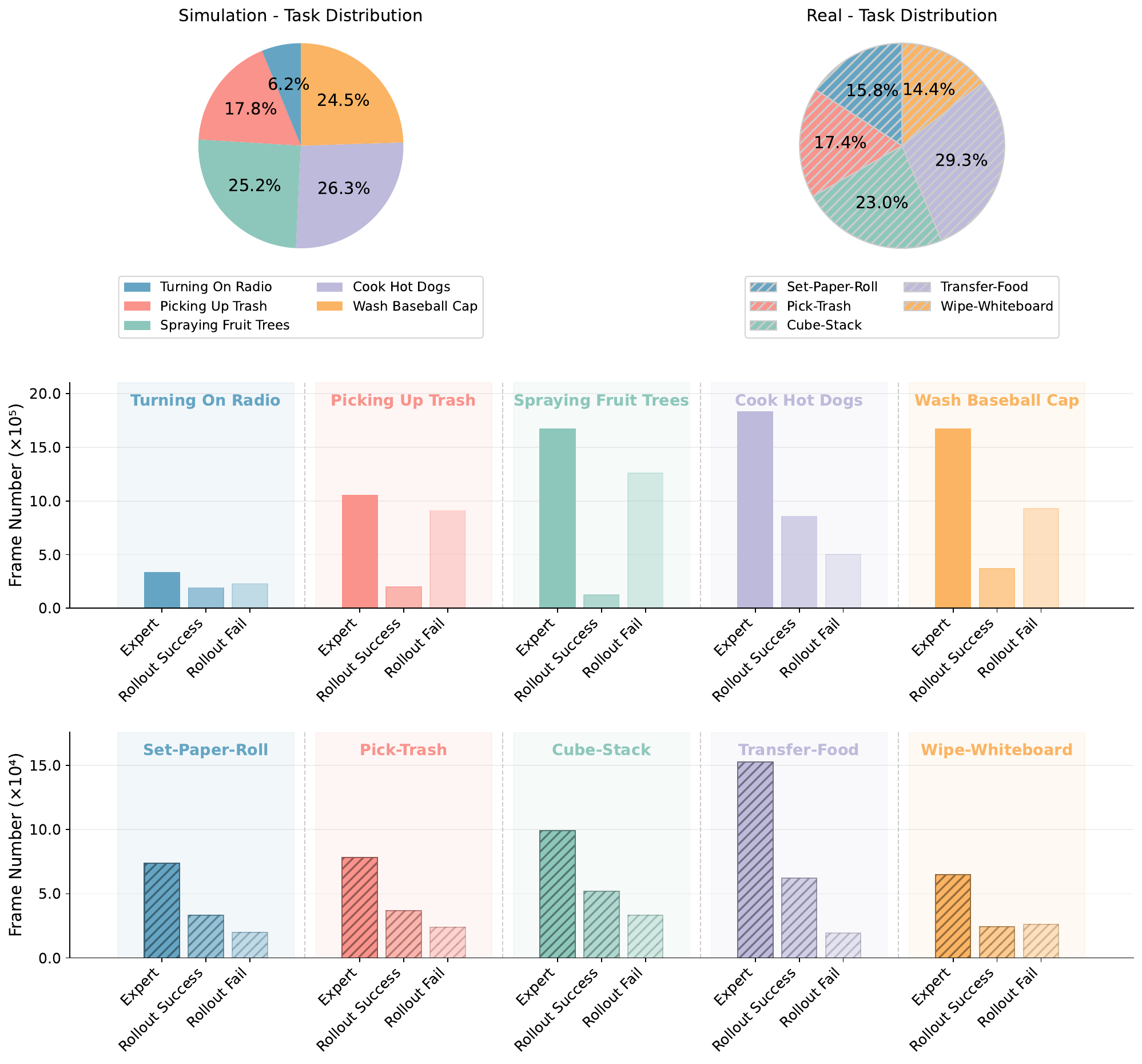}
    \caption{\textbf{Dataset composition.} Frame-level distribution of expert and autonomous rollout data across the five simulation and five real-world tasks, including successful and failed trajectories.}
    \label{fig:data_dist}
\end{figure}



\clearpage
\suppressfloats[t]
\section{Qualitative Visualization}
\label{potential_example}

To further verify the learned potential field $s$, we visualize the temporal evolution of the predicted success potential $s_t$ across five simulated tasks and five real-world tasks. For each task, both successful and failed trajectories are presented, with representative frames aligned to key stages along the potential curve.

As shown in Figs.~\ref{fig:cook_frames}--\ref{fig:wash_frames}, in the simulated tasks, successful trajectories generally exhibit clear local peaks around meaningful subgoal completions, such as opening articulated objects, grasping target objects, transferring items, or completing task-specific interactions. In contrast, failed trajectories typically show transient increases near partial progress, followed by sharp drops or persistently low potential after incorrect actions.

Similar patterns can be observed in the real-world tasks shown in Figs.~\ref{fig:real_build_blocks_frames}--\ref{fig:real_pick_trash_frames}. Successful executions produce stage-wise potential increases that align with semantically important manipulation steps, whereas failed executions lead to low or decreasing potential when the robot deviates from the correct task sequence. 

Overall, these qualitative results indicate that $s$ captures meaningful success-related semantics across both simulated and real-world manipulation tasks.

\FloatBarrier


\begin{figure}[H]
    \centering

    \begin{subfigure}{\linewidth}
        \centering
        \includegraphics[width=\linewidth]{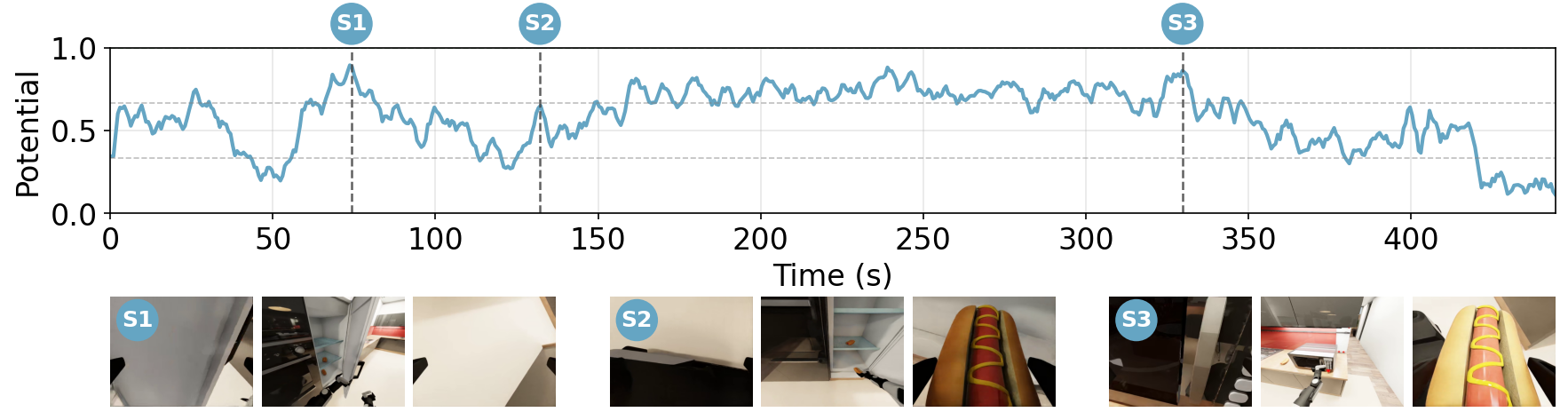}
        \caption{Success: S1 opens the refrigerator; S2 grasps the first hot dog; S3 opens the microwave, matching local potential peaks.}
        \label{fig:cook_success_frames}
    \end{subfigure}

    \vspace{0.05em}

    \begin{subfigure}{\linewidth}
        \centering
        \includegraphics[width=\linewidth]{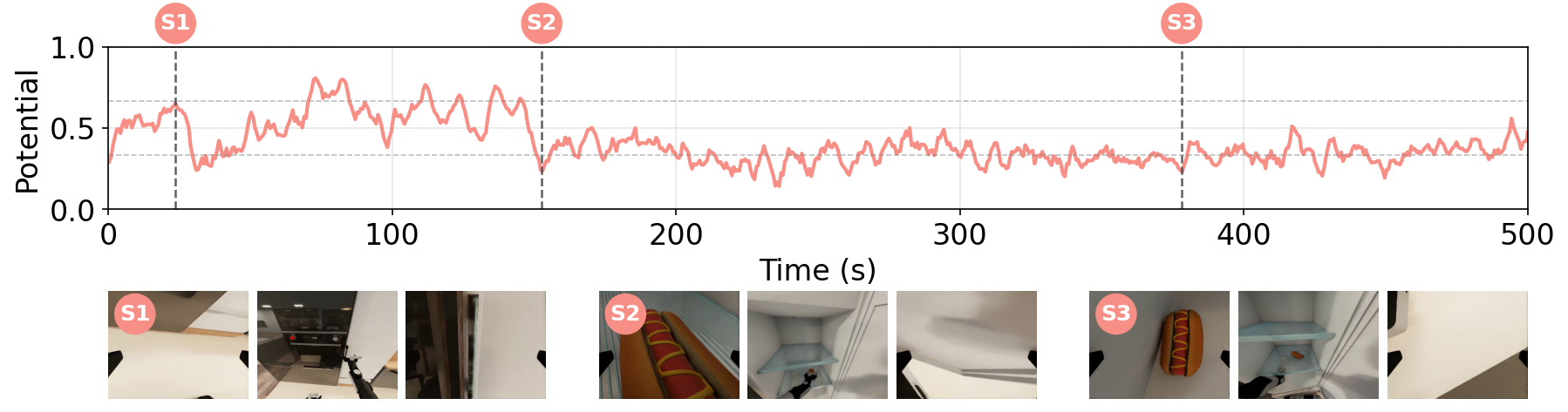}
        \caption{Failure: S1 opens the refrigerator; at S2, an incorrect grasping motion causes a sharp potential drop; around S3, repeated wrong attempts keep the potential low.}
        \label{fig:cook_failure_frames}
    \end{subfigure}

    \caption{Success and failure trajectory frames for the simulated \textit{Cook Hot Dogs} task.}
    \label{fig:cook_frames}
\end{figure}

\begin{figure*}[!t]
    \centering

    \begin{subfigure}{\linewidth}
        \centering
        \includegraphics[width=\linewidth]{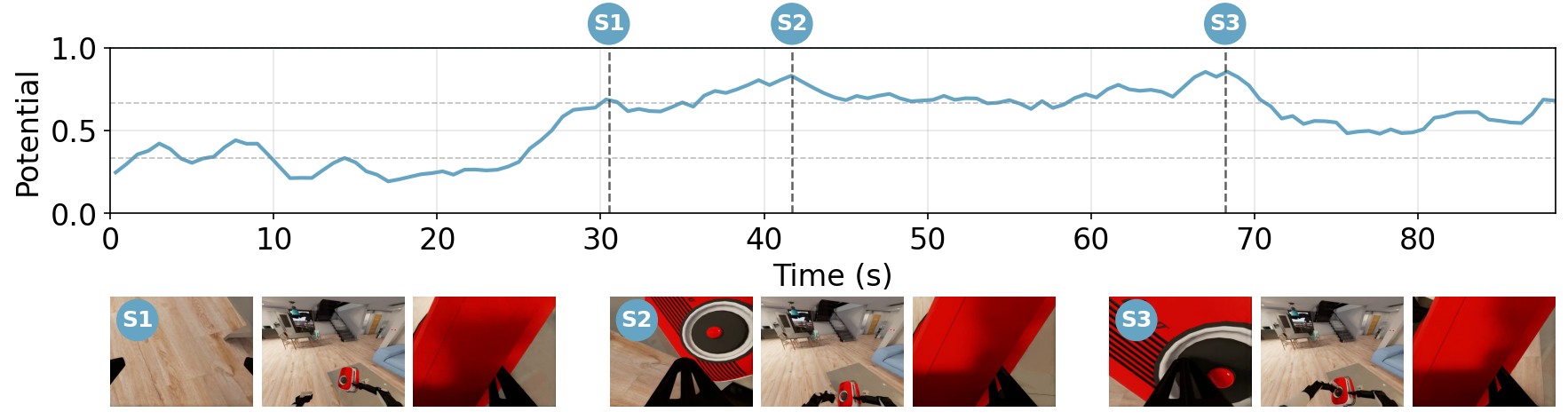}
        \caption{Success: S1 grasps the radio; S2 aligns the left hand with the power button; S3 turns on the radio, with local potential peaks.}
        \label{fig:radio_success_frames}
    \end{subfigure}

    \vspace{0.05em}

    \begin{subfigure}{\linewidth}
        \centering
        \includegraphics[width=\linewidth]{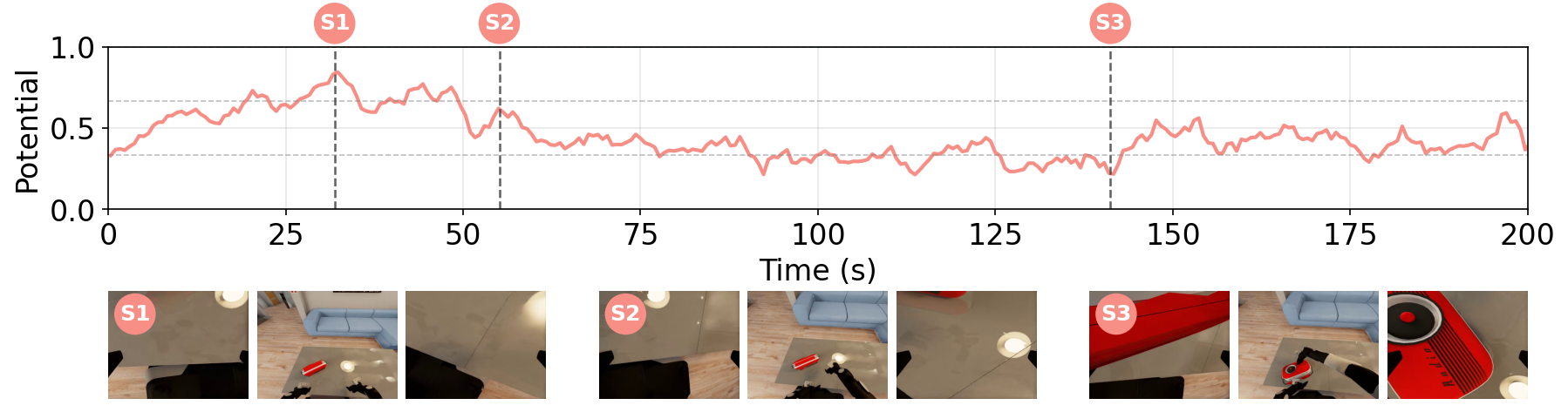}
        \caption{Failure: S1 reaches the radio; at S2, the robot knocks it over; around S3, repeated ineffective attempts keep the potential low.}
        \label{fig:radio_failure_frames}
    \end{subfigure}

    \caption{Success and failure trajectory frames for the simulated \textit{Turning on radio} task.}
    \label{fig:radio_frames}
\end{figure*}

\begin{figure*}[t]
    \centering

    \begin{subfigure}{\linewidth}
        \centering
        \includegraphics[width=\linewidth]{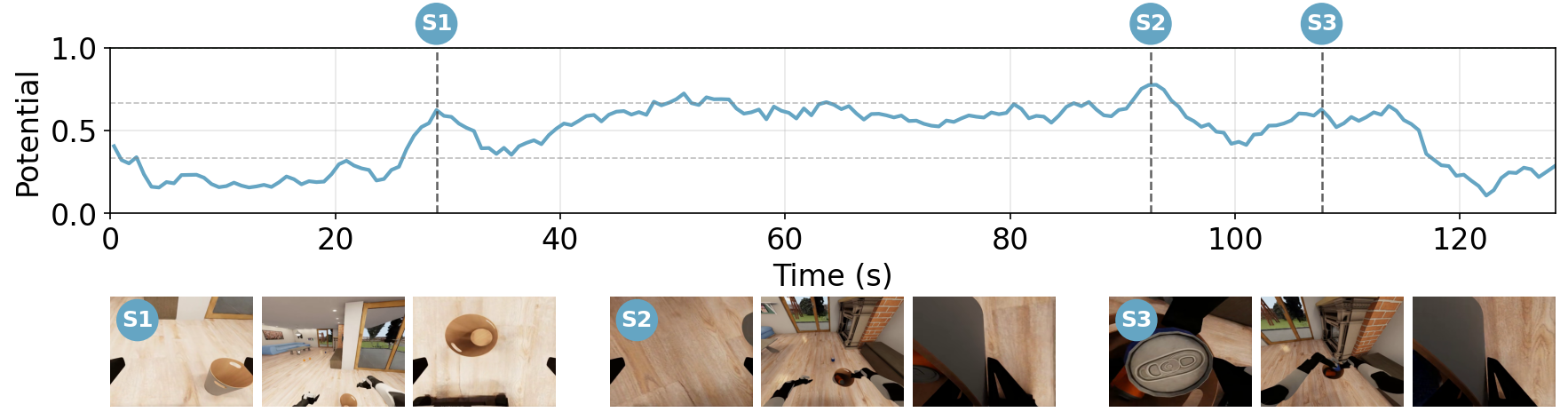}
        \caption{Success: S1 approaches the bin; S2--S3 grasp the two cans, matching local potential peaks.}
        \label{fig:trash_success_frames}
    \end{subfigure}

    \vspace{0.05em}

    \begin{subfigure}{\linewidth}
        \centering
        \includegraphics[width=\linewidth]{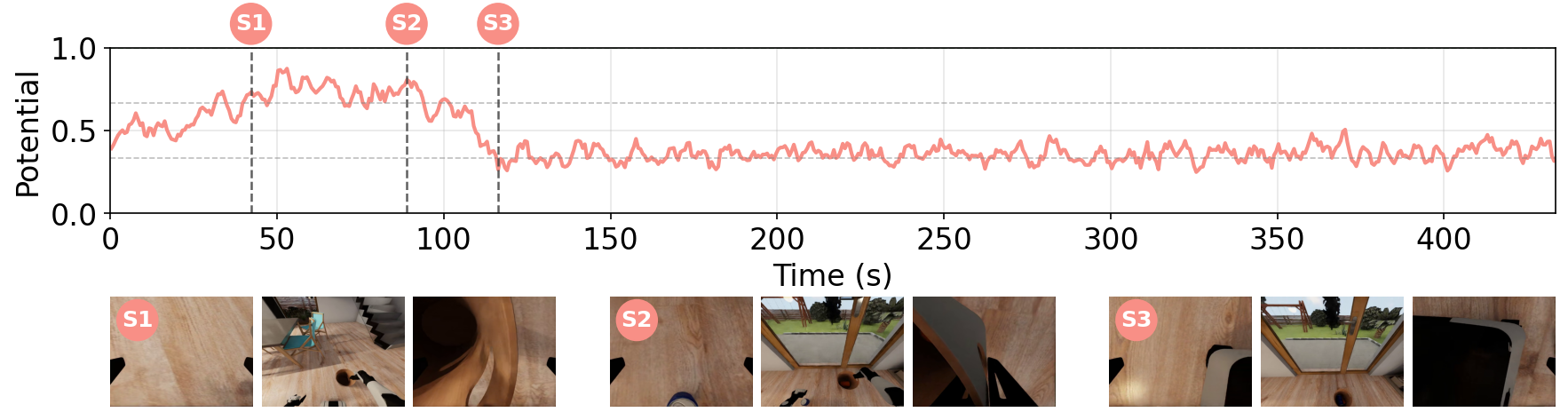}
        \caption{Failure: S1 approaches the bin; S2 approaches the first can; at S3, the robot wrongly places the bin on the floor and the potential stays low.}
        \label{fig:trash_failure_frames}
    \end{subfigure}

    \caption{Success and failure trajectory frames for the simulated \textit{Picking up trash} task.}
    \label{fig:trash_frames}
\end{figure*}

\begin{figure*}[t]
    \centering

    \begin{subfigure}{\linewidth}
        \centering
        \includegraphics[width=\linewidth]{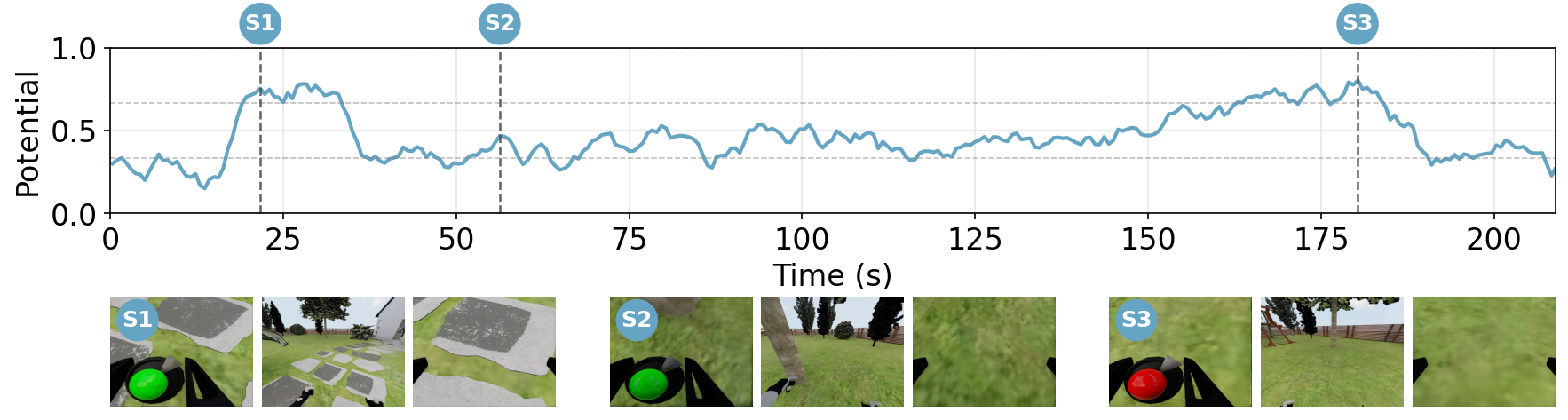}
        \caption{Success: S1 grasps the watering can; S2 waters the first tree; S3 approaches the second tree, with local potential peaks.}
        \label{fig:trees_success_frames}
    \end{subfigure}

    \vspace{0.05em}

    \begin{subfigure}{\linewidth}
        \centering
        \includegraphics[width=\linewidth]{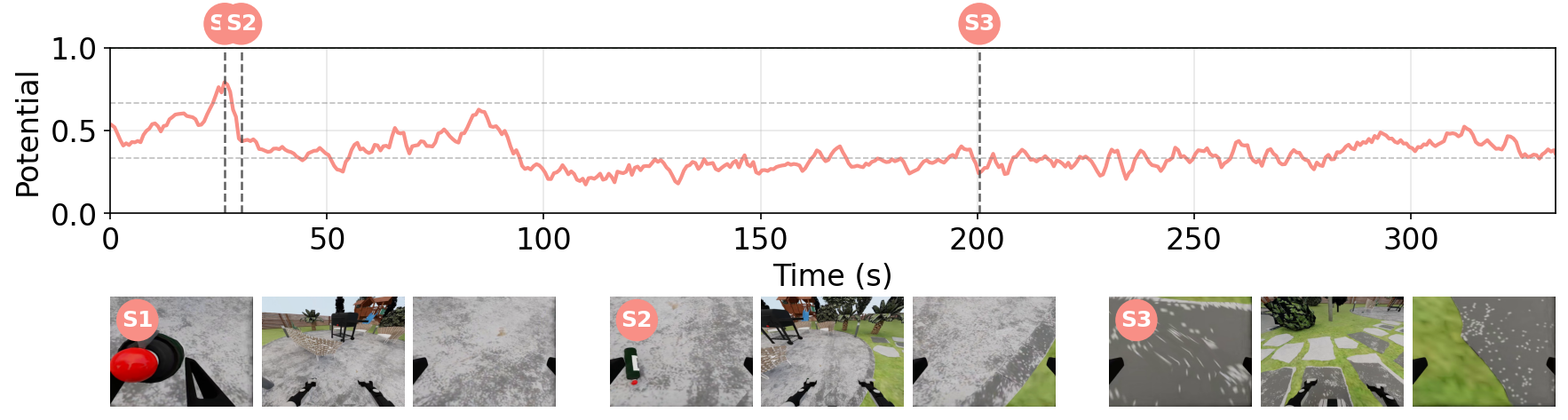}
        \caption{Failure: S1 grasps the watering can; at S2, the can is dropped; around S3, incorrect attempts keep the potential low.}
        \label{fig:trees_failure_frames}
    \end{subfigure}

    \caption{Success and failure trajectory frames for the simulated \textit{Spraying fruit trees} task.}
    \label{fig:trees_frames}
\end{figure*}

\begin{figure*}[t]
    \centering

    \begin{subfigure}{\linewidth}
        \centering
        \includegraphics[width=\linewidth]{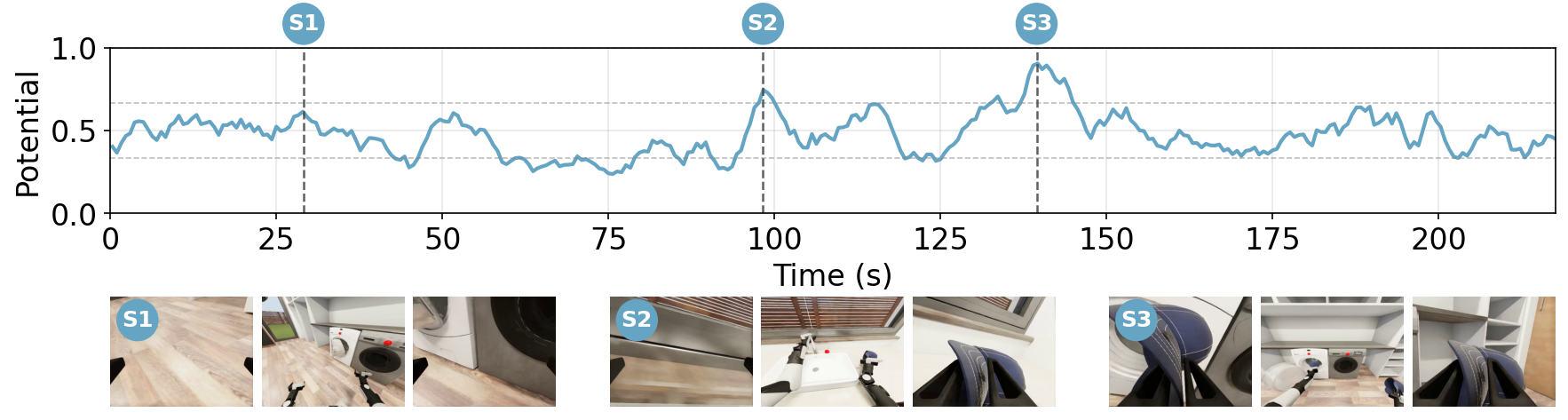}
        \caption{Success: S1 approaches the washing machine; S2 grasps the first cap; S3 carries two caps toward the machine, with local potential peaks.}
        \label{fig:wash_success_frames}
    \end{subfigure}

    \vspace{0.05em}

    \begin{subfigure}{\linewidth}
        \centering
        \includegraphics[width=\linewidth]{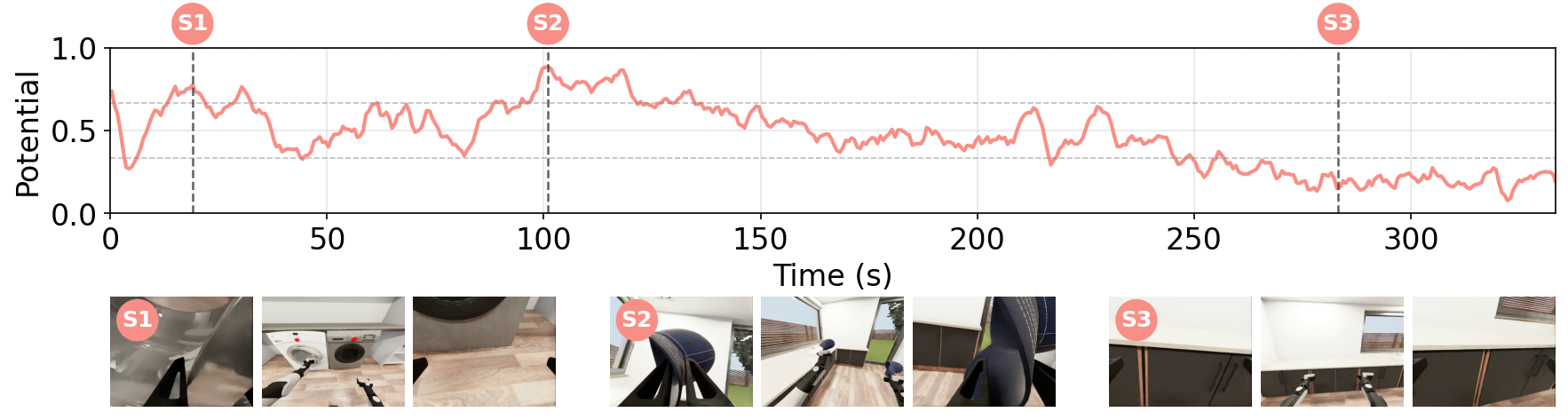}
        \caption{Failure: S1 opens the door; S2 grasps one cap; around S3, the robot fails to open the machine and searches incorrectly, keeping the potential low.}
        \label{fig:wash_failure_frames}
    \end{subfigure}

    \caption{Success and failure trajectory frames for the simulated \textit{Wash a baseball cap} task.}
    \label{fig:wash_frames}
\end{figure*}


\begin{figure*}[t]
    \centering

    \begin{subfigure}{\linewidth}
        \centering
        \includegraphics[width=\linewidth]{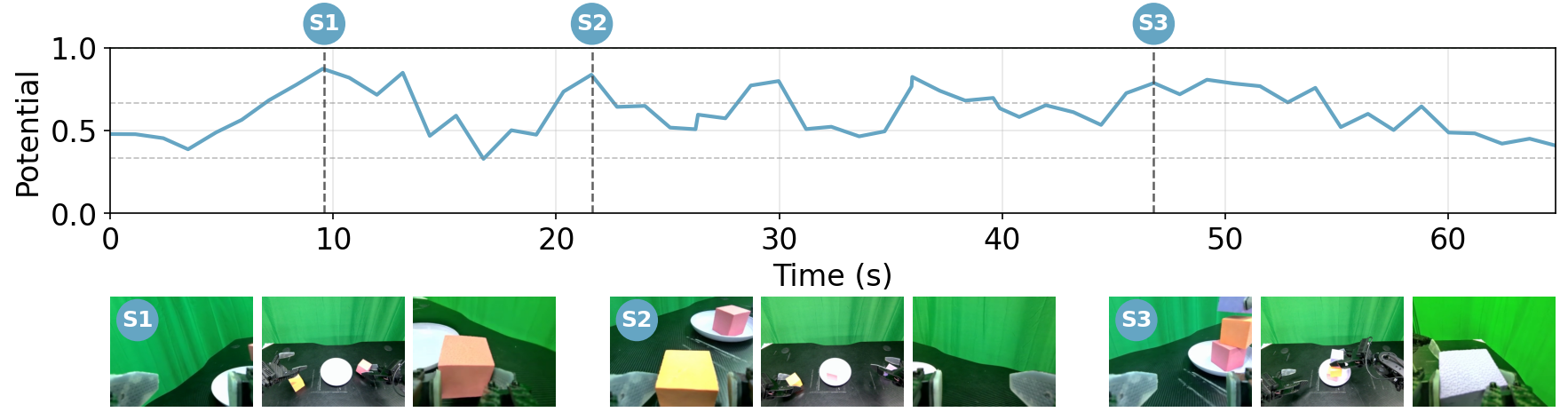}
        \caption{Success: S1 grasps the first block; S2 grasps the second block; S3 stacks the third block, with local potential peaks.}
        \label{fig:real_build_blocks_success_frames}
    \end{subfigure}

    \vspace{0.05em}

    \begin{subfigure}{\linewidth}
        \centering
        \includegraphics[width=\linewidth]{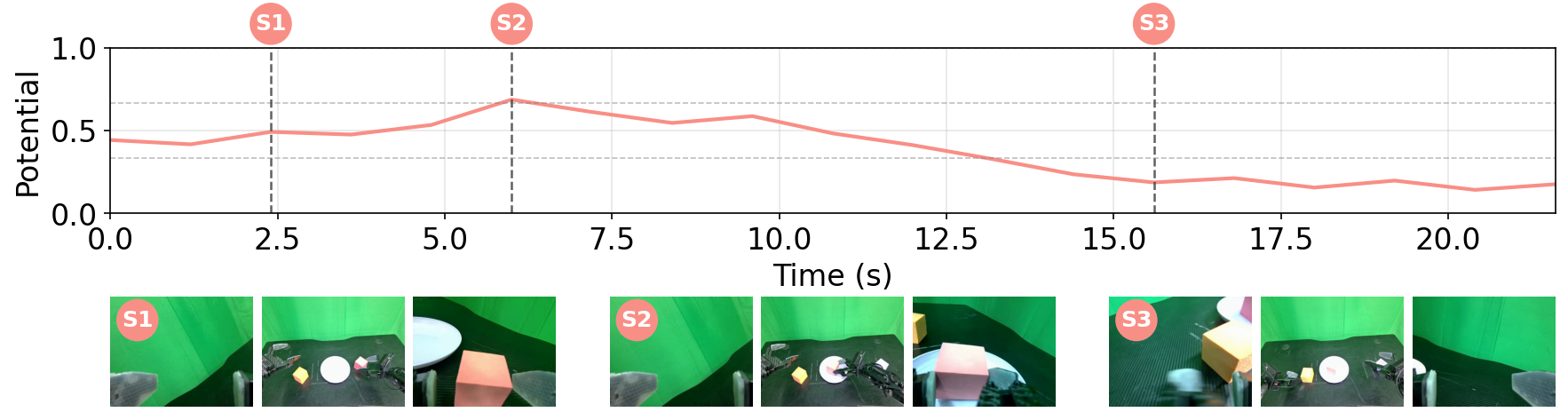}
        \caption{Failure: S1 grasps a block; S2 places it on the plate; around S3, failed grasp attempts keep the potential low.}
        \label{fig:real_build_blocks_failure_frames}
    \end{subfigure}

    \caption{Success and failure trajectory frames for the real-world \textit{Cube-Stack} task.}
    \label{fig:real_build_blocks_frames}
\end{figure*}

\begin{figure*}[t]
    \centering

    \begin{subfigure}{\linewidth}
        \centering
        \includegraphics[width=\linewidth]{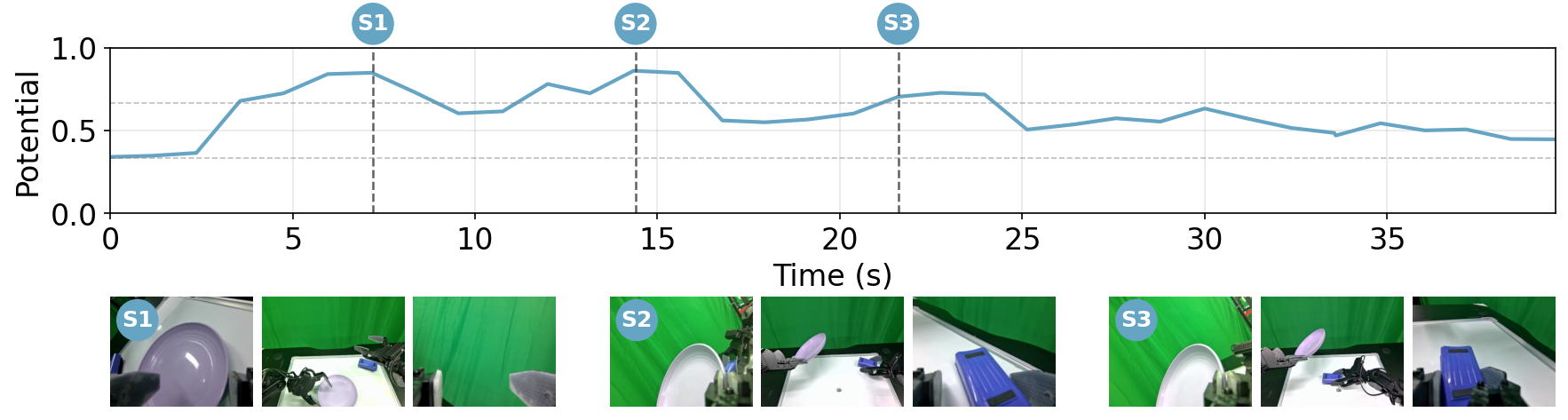}
        \caption{Success: S1 prepares to grasp the plate; S2 grasps the eraser and moves the bowl aside; S3 wipes the board, with local potential peaks.}
        \label{fig:real_clean_board_success_frames}
    \end{subfigure}

    \vspace{0.05em}

    \begin{subfigure}{\linewidth}
        \centering
        \includegraphics[width=\linewidth]{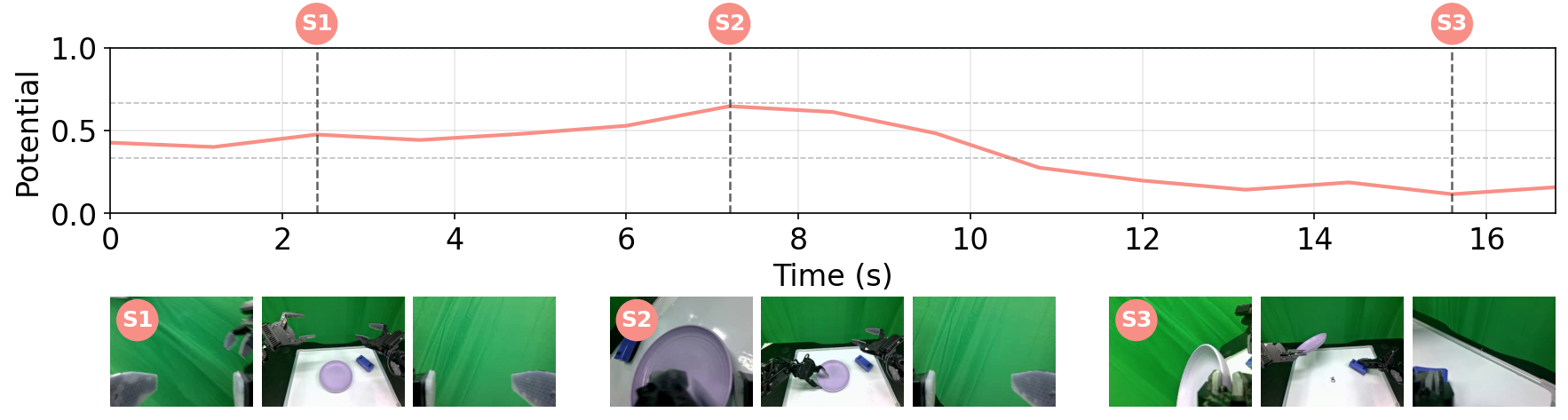}
        \caption{Failure: S1 orients toward the objects; S2 grasps the plate; around S3, the robot wipes the wrong area and the potential stays low.}
        \label{fig:real_clean_board_failure_frames}
    \end{subfigure}

    \caption{Success and failure trajectory frames for the real-world \textit{Wipe-Whiteboard} task.}
    \label{fig:real_clean_board_frames}
\end{figure*}

\begin{figure*}[t]
    \centering

    \begin{subfigure}{\linewidth}
        \centering
        \includegraphics[width=\linewidth]{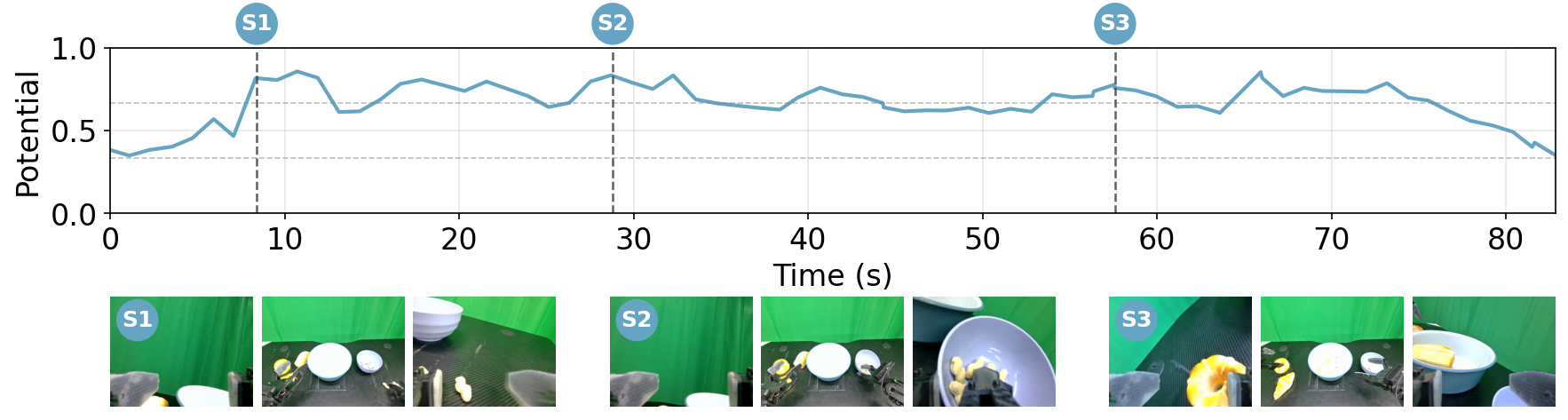}
        \caption{Success: S1 grasps the food item; S2 grasps the peanut bowl; S3 prepares to grasp the bread, with local potential peaks.}
        \label{fig:real_pick_food_success_frames}
    \end{subfigure}

    \vspace{0.05em}

    \begin{subfigure}{\linewidth}
        \centering
        \includegraphics[width=\linewidth]{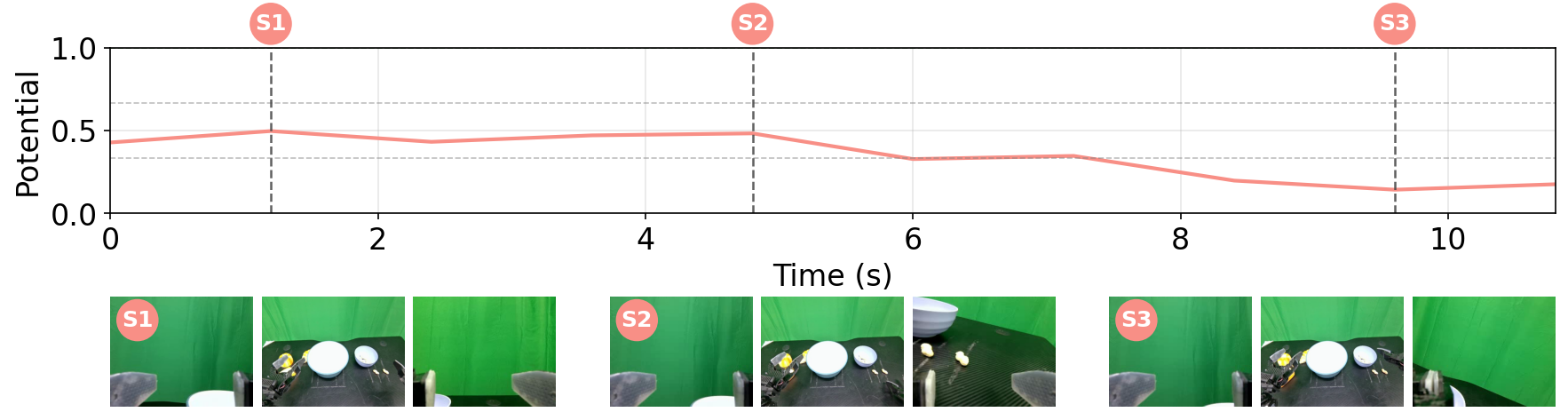}
        \caption{Failure: S1 orients toward the targets; around S2--S3, repeated failed peanut-grasp attempts drive the potential low.}
        \label{fig:real_pick_food_failure_frames}
    \end{subfigure}

    \caption{Success and failure trajectory frames for the real-world \textit{Transfer-Food} task.}
    \label{fig:real_pick_food_frames}
\end{figure*}

\begin{figure*}[t]
    \centering

    \begin{subfigure}{\linewidth}
        \centering
        \includegraphics[width=\linewidth]{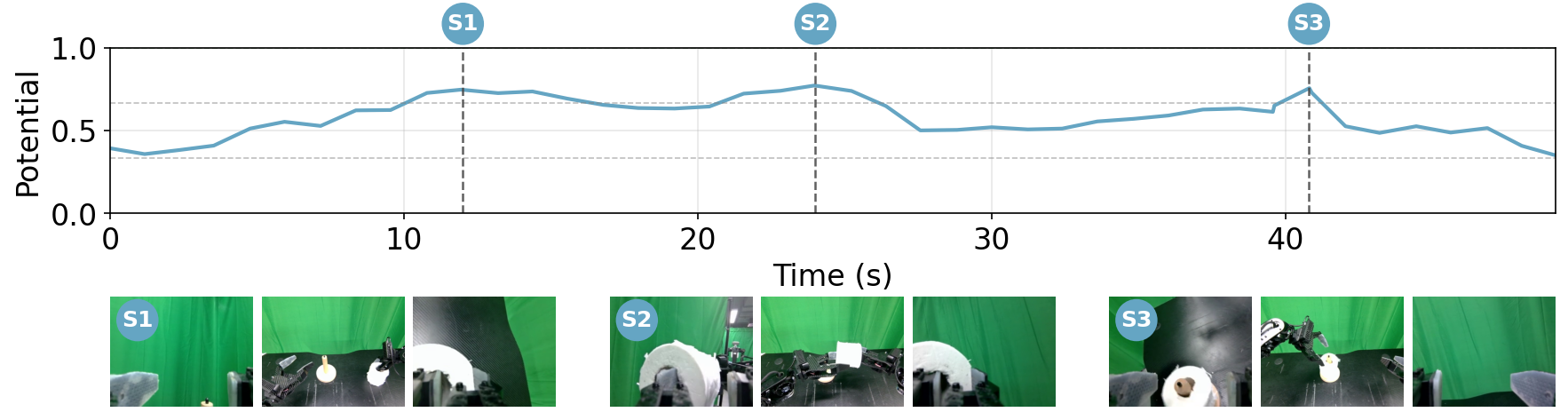}
        \caption{Success: S1 grasps the paper roll; S2 transfers it between hands; S3 places it onto the holder, with local potential peaks.}
        \label{fig:real_pick_paper_success_frames}
    \end{subfigure}

    \vspace{0.05em}

    \begin{subfigure}{\linewidth}
        \centering
        \includegraphics[width=\linewidth]{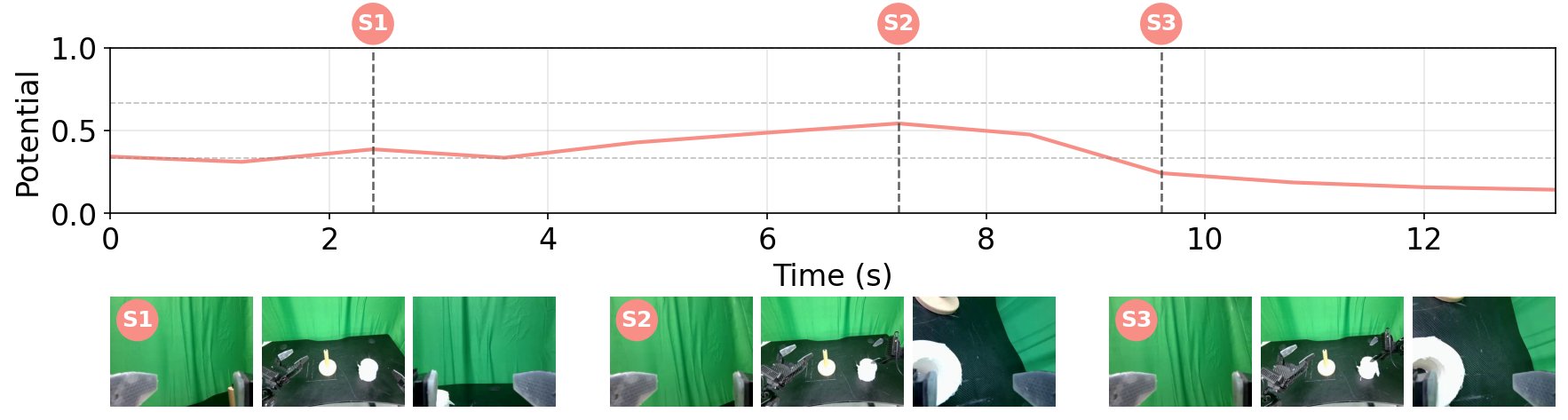}
        \caption{Failure: S1 orients toward the target; S2 moves close to the roll; around S3, the hand moves away and the potential drops.}
        \label{fig:real_pick_paper_failure_frames}
    \end{subfigure}

    \caption{Success and failure trajectory frames for the real-world \textit{Set-Paper-Roll} task.}
    \label{fig:real_pick_paper_frames}
\end{figure*}


\begin{figure*}[t]
    \centering

    \begin{subfigure}{\linewidth}
        \centering
        \includegraphics[width=\linewidth]{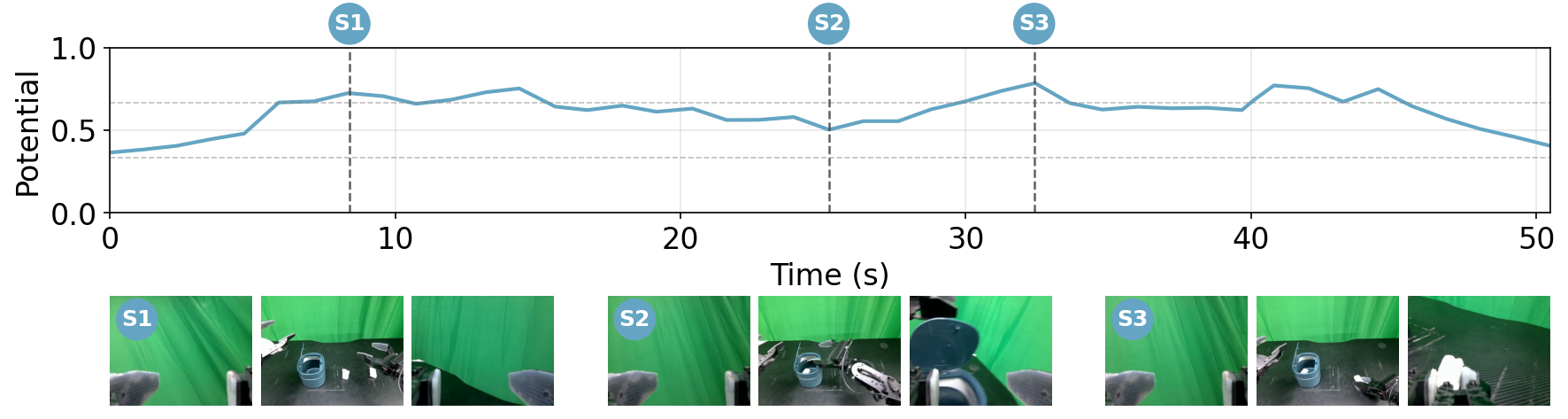}
        \caption{Success: S1 approaches the target; S2 adjusts after opening the bin; S3 grasps the trash, with a final local potential peak.}
        \label{fig:real_pick_trash_success_frames}
    \end{subfigure}

    \vspace{0.05em}

    \begin{subfigure}{\linewidth}
        \centering
        \includegraphics[width=\linewidth]{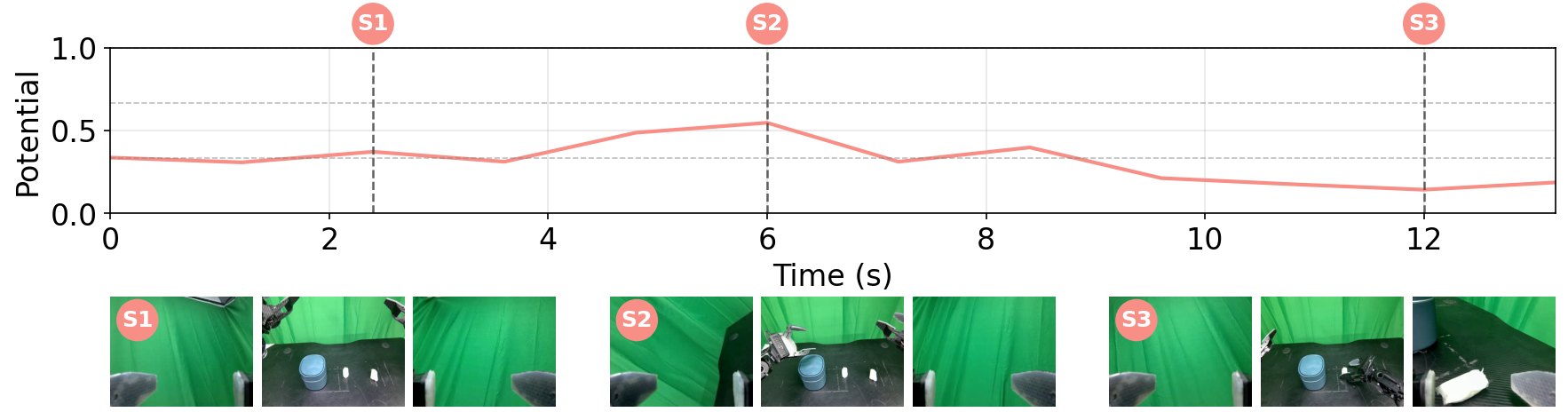}
        \caption{Failure: S1 orients toward the object; S2 moves close to the target; around S3, incorrect actions make the potential drop.}
        \label{fig:real_pick_trash_failure_frames}
    \end{subfigure}

    \caption{Success and failure trajectory frames for the real-world \textit{Pick-Trash} task.}
    \label{fig:real_pick_trash_frames}
\end{figure*}

\end{document}